\documentclass[11pt]{article}

\usepackage[margin=1in]{geometry}
\usepackage[utf8]{inputenc}
\usepackage{amsmath,amssymb,amsthm,mathtools}
\usepackage{graphicx}
\usepackage{microtype}
\usepackage{tikz}
\usetikzlibrary{calc}
\usepackage{booktabs}
\usetikzlibrary{positioning,fit,arrows.meta}
\usepackage{natbib}
\usepackage[colorlinks=true,linkcolor=blue,citecolor=blue,urlcolor=blue]{hyperref}
\usepackage{url}
\newtheorem{theorem}{Theorem}
\newtheorem{proposition}{Proposition}

\newtheorem{lemma}{Lemma}
\theoremstyle{definition}

\theoremstyle{remark}

\newcommand{\M}{\mathcal{M}}
\newcommand{\R}{\mathbb{R}}
\newcommand{\Sph}{\mathbb{S}}
\newcommand{\Hyp}{\mathbb{H}}
\newcommand{\E}{\mathbb{E}}
\newcommand{\dd}{\mathrm{d}}
\newcommand{\PT}[2]{\Gamma_{#1\to#2}}
\newcommand{\KL}{\mathrm{KL}}

\newcommand{\Ng}{\mathcal{N}_g}
\DeclareMathOperator{\gradg}{grad_g}
\DeclareMathOperator{\divg}{div_g}
\DeclareMathOperator{\vol}{vol_g}
\DeclareMathOperator{\Hess}{Hess_g}
\DeclareMathOperator{\pr}{pr}
\newcommand{\N}{\mathcal{N}}
\newcommand{\Flow}{\Phi}
\newcommand{\dvol}{\,\mathrm{d}\vol}
\newcommand{\normg}[1]{\lVert #1\rVert_{g^{-1}}}

\newcommand{\Kk}{\mathsf K}
\newcommand{\Pv}{\pi_{V,\tau}}

\newcommand{\eqn}[1]{(\ref{#1})}

\title{Riemannian Neural Hamiltonian Flows:\\
From Interpretable Potentials to Exact Surrogate Sampling}
\author{Vincent Souveton}
\date{CEA, DAM, DIF, Arpajon, F-91297, France}

\begin{document}
\maketitle

\begin{abstract}
\noindent Hamiltonian normalizing flows are attractive generative models because their phase-space maps are invertible and volume preserving, but most neural constructions are formulated in Euclidean space. We introduce Riemannian Neural Hamiltonian Flows, which combine the fixed kinetic energy of a Riemannian manifold, a learned scalar potential, and an explicit geodesic integrator. We then study the interpretability of the learned potential through a thermodynamic reading in which the momentum variance plays the role of a temperature. The matched potential, whose Boltzmann distribution is the data, is the potential whose force balances the pressure of the data. The output of one pass of the flow is the Boltzmann distribution of the learned potential multiplied by a memory factor, so that the learned potential is the matched potential whenever the pass forgets its starting point. Numerical experiments on Euclidean, hyperbolic, and spherical spaces show competitive sample quality and numerical cost against a Riemannian continuous normalizing flow, and confirm the interpretability of the learned potential. We then show that this reading has a computational use. When the energy of the data is known but expensive, the landscape learned from existing samples is a surrogate energy for Hamiltonian Monte Carlo, allowing exact sampling with significant acceleration. We test this approach on the localization of a neutrino source from public IceCube data. The code is available at \url{https://github.com/vincentsouveton/RiemannianNHF}.
\end{abstract}

\section{Introduction}
\label{sec:intro}

Learning expressive probability distributions is a central problem in generative modeling \citep{goodfellow2016}. Different classes of generative models approach this problem through different parameterizations and learning objectives. Variational autoencoders introduce a latent-variable model trained through a variational lower bound on the log-likelihood \citep{kingma2014}. Generative adversarial networks are trained through the competition between a discriminator, whose role is to discriminate between real and generated data, and a generator, whose role is to trick the discriminator \citep{goodfellow2014}. Denoising diffusion and score-based models estimate the score fields of a family of progressively perturbed distributions \citep{sohl-dickstein2015, ho2020, song2021}. In this work, we instead rely on flow-based models that transport a base distribution to the target using an invertible map, allowing efficient and stable likelihood-based training. Among existing approaches, normalizing flows rely on a finite sequence of smooth bijections and density evaluation follows from the change-of-variables formula \citep{tabak2010, rezende2015, dinh2017, papamakarios2021}. Continuous normalizing flows replace the discrete sequence by an ordinary differential equation and recover density changes by integrating the divergence of the learned vector field \citep{chen2018, grathwohl2019}. Flow matching instead learns a continuous transport from prescribed probability paths, without explicitly integrating the corresponding density equation \citep{lipman2023, liu2023, albergo2025}. These constructions provide flexible mechanisms for transporting probability mass, but likelihood evaluation requires costly Jacobian determinants or divergence estimates, and most formulations are suited to Euclidean spaces.

Many data sets live instead on curved spaces: directions on spheres \citep{mardia2000}, hierarchical representations of symbolic data in hyperbolic spaces \citep{nickel2017}, or periodic variables on tori. Applying a Euclidean flow to such data may distort intrinsic distances and volumes, violate support constraints, or introduce coordinate singularities. This has motivated chart- and embedding-based flows \citep{gemici2016, rezende2020}, intrinsic continuous flows \citep{lou2020, mathieu2020}, divergence-based constructions \citep{rozen2021}, Riemannian score-based diffusions \citep{debortoli2022} and Riemannian flow matching \citep{chen2024}. These methods differ in their objectives and parameterizations, but they all require the learned transport to account explicitly for the geometry-dependent change of volume.

Hamiltonian dynamics provides a complementary route. A Hamiltonian flow acts on an extended phase space of positions and momenta and preserves its canonical Liouville volume. Lifting an observed variable q to a phase-space state (q, p) therefore yields an invertible transformation with unit phase-space Jacobian. For sampling purposes, Hamiltonian dynamics was introduced in computational statistics as a proposal mechanism for Markov chain Monte Carlo \citep{robert2004}. Hamiltonian Monte Carlo augments the variable of interest with an auxiliary momentum and follows approximately energy-preserving trajectories to propose distant moves, while a Metropolis correction preserves the target distribution \citep{duane1987, neal2011, betancourt2017}. Riemannian variants adapt the kinetic energy to the local geometry \citep{girolami2011, byrne2013}. More recently, Hamiltonian maps have also been used as generative models rather than as components of a Monte Carlo transition \citep{caterini2018, toth2020, holderrieth2024}. In particular, Neural Hamiltonian Flows (NHF) learn a scalar Hamiltonian, which yields a volume-preserving generative map whose learned energy landscape can be inspected directly \citep{souveton2024}. However, these constructions have been developed in Euclidean spaces.

This leaves three questions open. First, can Hamiltonian generative models be extended intrinsically to Riemannian manifolds, while respecting the geometry of positions, momenta, geodesics and volumes? Second, do the resulting models retain the practical properties observed in the Euclidean setting, namely competitive sample quality, a moderate computational cost and an interpretable learned potential? And finally, can we take advantage of this interpretability property to build numerically efficient samplers? This paper addresses these questions through Riemannian Neural Hamiltonian Flows (RNHF), an intrinsic extension of NHF to non-Euclidean geometries. Our contributions are threefold.

First, \textbf{a Riemannian Hamiltonian generative model.} RNHF combines the kinetic energy induced by the metric, Gaussian momenta defined on each cotangent space, a learned scalar potential and a geodesic kick--drift--kick integrator. The resulting map stays on the manifold, preserves phase-space volume exactly and can be inverted exactly, so training requires neither Jacobian determinants nor Riemannian divergences. Second, \textbf{a framework for interpretability with application to Markov Chaine Monte Carlo sampling.} Viewing the momentum variance as a temperature and the lifted base as a gas, we call matched potential the potential whose Boltzmann distribution is the data, known in statistical physics as the free energy, and whose force balances the pressure of the data. One pass of the flow outputs the Boltzmann distribution of the learned potential multiplied by a memory factor, which gives an exact expression of the gap between the learned potential and the matched potential, and shows that a pass that forgets its starting point learns the matched potential. Finally, when the energy of the data is known but expensive, the landscape learned from existing samples can be turned into a surrogate energy for Hamiltonian Monte Carlo. Third, \textbf{a reproducible numerical study.} On Euclidean, hyperbolic, and spherical spaces, RNHF is competitive with a Riemannian continuous normalizing flow and its learned landscapes are close to the matched ones. We describe an experiment on the localization of a neutrino source on the sky from about a million public IceCube events, to test the surrogate sampler on a problem whose energy is genuinely expensive.

The paper is organized as follows. Section~\ref{sec:model} introduces the geometric tools and the RNHF model. Section~\ref{sec:thermo} studies the interpretability of the learned potential and its use as a surrogate energy. Section~\ref{sec:experiments} reports the numerical experiments, and Section~\ref{sec:conclusion} concludes. Proofs and experimental details are given in the appendices.

\section{Riemannian Neural Hamiltonian Flows}
\label{sec:model}

We are given samples $q^{(1)},\dots,q^{(n)}$ drawn from an unknown target density $\rho_1$ on a space $\M$, and seek a model that generates new samples from that density. In this work, $\M$ is a curved space, so distances, volumes and supports must be defined intrinsically. A generative model built in an ambient Euclidean chart may place mass off the manifold, distort volumes, or consider close two points that are geodesically far apart. This section first introduces the geometric tools and the Hamiltonian dynamics we use (Section~\ref{sec:hamiltonian}), then recalls the Euclidean Neural Hamiltonian Flows on which our construction is based (Section~\ref{sec:nhf}), and finally defines the Riemannian model (Section~\ref{sec:rnhf}).

\subsection{Hamiltonian dynamics on a Riemannian manifold}
\label{sec:hamiltonian}

A Riemannian manifold $(\M,g)$ comes with a metric $g$, which is, at each point $q$, an inner product on the tangent plane $T_q\M$, the space of velocities at $q$. The metric measures lengths, angles and volumes. In particular, the Riemannian volume $\vol$ plays the role of the Lebesgue measure, and every density in this paper is taken with respect to it. Three further tools, illustrated in Figure~\ref{fig:geometry}, replace the straight lines of Euclidean space. A geodesic is the curve followed by a free particle, and the exponential map $\exp_q(v)$ is the point reached by following the geodesic that leaves $q$ with velocity $v$ for unit time. Parallel transport $\PT{q}{q'}$ carries a vector from $q$ to $q'$ along the geodesic while keeping it as constant as the geometry allows, without changing its norm. Finally, the volume Jacobian $J_o$ of the exponential map compares the volume of a small region of the flat tangent plane $T_o\M$ with the volume of its image on $\M$. It turns a density $\mu$ defined on $T_o\M$ into a density on the manifold: if $X\sim\mu$, then $Q=\exp_o(X)$ has density $\mu(x)/J_o(x)$ at $\exp_o(x)$, a construction known as wrapping \citep{nagano2019}. On the sphere, $J_o(x)=\sin\lVert x\rVert/\lVert x\rVert$, and on the hyperbolic plane, $J_o(x)=\sinh\lVert x\rVert/\lVert x\rVert$ grows exponentially, because hyperbolic space contains exponentially more room far from any point. Appendix~\ref{app:geometry} gives formal definitions, and Table~\ref{tab:primitives} gives closed forms on the three model spaces of our experiments.

\begin{figure}[!t]
\centering
\resizebox{\linewidth}{!}{\begin{tikzpicture}[
  >={Latex[length=2.2mm]},
  vec/.style={->, very thick},
  lbl/.style={font=\small},
  plane/.style={fill=gray!12, draw=gray!55, fill opacity=0.6}
]
\begin{scope}
  \shade[ball color=blue!5, opacity=0.9] (0,0) circle (2.2);
  \draw[gray!70] (0,0) circle (2.2);
  \draw[gray!40, dashed] (-2.2,0) arc (180:360:2.2 and 0.55);
  \coordinate (q)  at (-1.15,1.15);
  \coordinate (qp) at (1.2,0.95);
  \filldraw[plane] ($(q)+(-0.7,-0.45)$) -- ($(q)+(0.45,-0.75)$) -- ($(q)+(0.8,0.3)$) -- ($(q)+(-0.35,0.6)$) -- cycle;
  \filldraw[plane] ($(qp)+(-0.45,-0.7)$) -- ($(qp)+(0.7,-0.4)$) -- ($(qp)+(0.35,0.6)$) -- ($(qp)+(-0.8,0.3)$) -- cycle;
  \node[lbl, gray!50!black] at ($(q)+(-0.95,0.5)$) {$T_q\M$};
  \node[lbl, gray!50!black] at ($(qp)+(0.95,0.5)$) {$T_{q'}\M$};
  \draw[thick, red!70!black] (q) .. controls (-0.45,1.8) and (0.5,1.75) .. (qp);
  \node[lbl, red!70!black] at (0.02,1.95) {geodesic};
  \draw[vec, blue!70!black] (q) -- ($(q)+(0.6,0.42)$);
  \node[lbl, blue!70!black] at ($(q)+(0.2,0.6)$) {$v$};
  \draw[vec, blue!70!black] (qp) -- ($(qp)+(0.6,-0.42)$);
  \fill (q) circle (1.6pt);
  \node[lbl, anchor=north] at ($(q)+(0,-0.08)$) {$q$};
  \fill (qp) circle (1.6pt);
  \node[lbl, anchor=north east] at ($(qp)+(0.0,-0.12)$) {$q'=\exp_q(v)$};
  \node[lbl, blue!70!black, anchor=north] at ($(qp)+(0.9,-0.66)$) {$\Gamma_{q\to q'}v$};
  \node[lbl] at (0,-2.75) {(a) Exponential map, parallel transport};
\end{scope}

\begin{scope}[xshift=6.3cm, yshift=-0.45cm]
  \shade[ball color=blue!5, opacity=0.9] (0,0) circle (1.9);
  \draw[gray!70] (0,0) circle (1.9);
  \fill[orange!70, opacity=0.35] (0,1.684) ellipse (0.867 and 0.156);
  \draw[orange!80!black, thick] (0,1.684) ellipse (0.867 and 0.156);
  \draw[orange!80!black, thick] (0,1.267) ellipse (1.418 and 0.255);
  \fill (0,1.9) circle (1.4pt);
  \node[lbl, anchor=west] at (0.05,1.98) {$o$};
  \filldraw[plane] (-2.1,2.85) -- (1.7,2.85) -- (2.1,3.35) -- (-1.7,3.35) -- cycle;
  \node[lbl, gray!50!black] at (-1.75,3.6) {$T_o\M$};
  \fill[orange!70, opacity=0.35] (0,3.1) ellipse (0.9 and 0.11);
  \draw[orange!80!black, thick] (0,3.1) ellipse (0.9 and 0.11);
  \draw[orange!80!black, thick] (0,3.1) ellipse (1.6 and 0.2);
  \fill (0,3.1) circle (1.2pt);
  \draw[gray!60, densely dotted] (0,3.1) -- (0,1.9);
  \draw[->, dashed, gray!40!black] (1.6,3.05) -- (1.43,1.33);
  \draw[->, dashed, gray!40!black] (-1.6,3.05) -- (-1.43,1.33);
  \draw[->, dashed, gray!40!black] (0.9,3.05) -- (0.87,1.74);
  \node[lbl, gray!30!black, anchor=west] at (1.55,2.25) {$\exp_o$};
  \node[lbl, align=center] at (0,-2.3) {(b) Wrapping a tangent density:\\ circles shrink, $J_o(x)=\frac{\sin\lVert x\rVert}{\lVert x\rVert}$};
\end{scope}

\begin{scope}[xshift=11.3cm, yshift=-1.5cm]
  \draw[thick] (-1.4,0) .. controls (0,0.35) and (1.6,-0.35) .. (3.4,0.1) node[right, lbl] {$\M$};
  \foreach \x/\y in {-0.4/0.12, 2.5/-0.06} {
    \draw[gray!70, dashed] (\x,\y-0.3) -- (\x,\y+3.7);
  }
  \node[lbl, gray!50!black] at (-0.4,4.05) {$T^*_{q^n}\M$};
  \node[lbl, gray!50!black] at (2.5,4.05) {$T^*_{q^{n+1}}\M$};
  \coordinate (A) at (-0.4,0.9);
  \coordinate (B) at (-0.4,2.1);
  \coordinate (C) at (2.5,2.0);
  \coordinate (D) at (2.5,2.95);
  \draw[vec, orange!80!black] (A) -- (B) node[midway, left, lbl, align=right, orange!80!black] {half-kick\\ $-\tfrac h2\,\dd V(q^n)$};
  \draw[vec, red!70!black] (B) .. controls (0.5,2.85) and (1.6,2.75) .. (C);
  \node[lbl, red!70!black, align=center] at (1.05,3.35) {drift: $\exp$ and $\Gamma$};
  \draw[vec, orange!80!black] (C) -- (D);
  \node[lbl, orange!80!black, align=left, anchor=north west] at ($(C)+(0.08,0.0)$) {half-kick};
  \fill (A) circle (1.6pt) node[right, lbl] {$(q^n,p^n)$};
  \fill (D) circle (1.6pt) node[right, lbl] {$(q^{n+1},p^{n+1})$};
  \fill (-0.4,0.12) circle (1.3pt) node[below, lbl] {$q^n$};
  \fill (2.5,-0.06) circle (1.3pt) node[below, lbl] {$q^{n+1}$};
  \node[lbl] at (1.0,-1.3) {(c) One integrator step in $T^*\M$};
\end{scope}
\end{tikzpicture}}
\caption{Geometric tools. (a) Velocities at $q$ live in the tangent plane $T_q\M$. The exponential map follows the geodesic that leaves $q$ with velocity $v$, and parallel transport carries $v$ along this geodesic to the tangent plane at the arrival point, without rotating or stretching it. (b) A density defined on the flat tangent plane $T_o\M$ is wrapped onto the manifold by the exponential map. On the sphere, a circle of radius $r$ around $o$ becomes a circle of circumference $2\pi\sin r$ instead of $2\pi r$: the same probability mass occupies less room, and the density is divided by the volume Jacobian $J_o$. (c) Phase space is drawn as the manifold with one momentum fibre above each point. One step of the integrator~\eqref{eq:kdk} moves vertically in the fibre (half-kick), then along a geodesic while transporting the momentum (drift), then vertically again.}
\label{fig:geometry}
\end{figure}

In mechanics, a momentum is a covector, that is, a linear form on velocities, and it lives in the cotangent space $T^*_q\M$. The metric identifies a momentum $p$ with the velocity $p^\sharp=g_q^{-1}p$, and a phase-space state is a position--momentum pair $(q,p)$ in the cotangent bundle $T^*\M$, which gathers all cotangent spaces (Figure~\ref{fig:geometry}c). We consider the classical Hamiltonian
\begin{equation}
  H(q,p)=K(q,p)+V(q),\qquad K(q,p)=\tfrac12\,p^{\top}G(q)^{-1}p=\tfrac12\normg{p}^2,
  \label{eq:H}
\end{equation}
the sum of a quadratic kinetic energy determined by the metric, whose matrix in local coordinates is $G(q)$, and of a scalar potential $V:\M\to\R$. In local coordinates, the dynamics follows Hamilton's equations
\begin{equation}
  \dot q=G(q)^{-1}p,\qquad
  \dot p=-\tfrac12\,\tfrac{\partial}{\partial q}\bigl(p^{\top}G(q)^{-1}p\bigr)-\mathrm{d}V(q).
  \label{eq:hameq}
\end{equation}
Equivalently, a particle of unit mass obeys the covariant Newton equation $\nabla_t\dot q=-\gradg V(q)$, where $\nabla_t\dot q$ is the acceleration measured along the manifold and $\gradg V$ is the gradient measured with the metric. The flow $\Flow_t$ conserves the energy $H$ and the canonical Liouville volume $\lambda$ of $T^*\M$, which factorizes as $\dd\lambda=\dvol\,\dd\nu_q$, where $\dd\nu_q$ is the Euclidean volume induced by $g_q^{-1}$ on $T^*_q\M$ \citep{landau1976}.

Numerically, we integrate~\eqref{eq:hameq} with the geodesic integrator of \citet{byrne2013}. With step size $h$, one kick--drift--kick (KDK) step reads
\begin{equation}
  \begin{cases}
    p^{n+1/2}=p^{n}-\tfrac h2\,\mathrm{d}V(q^{n}),\\[2pt]
    q^{n+1}=\exp_{q^{n}}\!\bigl(h\,g^{-1}_{q^{n}}p^{n+1/2}\bigr),\qquad
    \widetilde p^{\,n+1/2}=\PT{q^{n}}{q^{n+1}}\,p^{n+1/2},\\[2pt]
    p^{n+1}=\widetilde p^{\,n+1/2}-\tfrac h2\,\mathrm{d}V(q^{n+1}).
  \end{cases}
  \label{eq:kdk}
\end{equation}
The scheme is symplectic and exactly volume-preserving for every $h$. It is also reversible, by applying the inverse substeps in reverse order, and second-order accurate \citep{hairer2006}. Because positions are only updated through the exponential map, trajectories never leave the manifold.

To draw the initial momentum, we need an intrinsic notion of Gaussian. For a scale $\tau>0$, we define the centered isotropic distribution
\begin{equation}
  M_\tau(q,p)=\frac{1}{(2\pi\tau)^{d/2}}\exp\Bigl(-\frac{\normg{p}^2}{2\tau}\Bigr),
  \qquad\text{written}\qquad P\mid Q=q\ \sim\ \Ng(0,\tau I),
  \label{eq:maxwellian}
\end{equation}
which is a probability density with respect to $\dd\nu_q$ for every $q$. In any orthonormal basis of $T^*_q\M$, the components of $p$ are independent centered Gaussians of variance $\tau$. On the sphere, such a momentum is obtained by projecting an isotropic Gaussian of $\R^3$ onto the plane orthogonal to $q$. On the hyperbolic plane, it is drawn in the tangent space at the origin and parallel-transported to $q$.

\subsection{Neural Hamiltonian Flows on Euclidean spaces}
\label{sec:nhf}

Normalizing flows push a base distribution $\rho_0$ onto the target $\rho_1$ through a learned sequence of smooth bijections $\Phi = \Phi_L \circ \cdots \circ \Phi_1$ \citep{papamakarios2021}. Training learns the map from data to base, and sampling draws from the base and applies the inverse map. The density of the model follows from the change-of-variables formula
\begin{equation}
  \rho_\theta(s)=\rho_0\bigl(\Phi^{-1}(s)\bigr)\,\bigl|\det\mathrm{Jac}\,\Phi^{-1}(s)\bigr|,
  \label{eq:nf}
\end{equation}
and the model is trained by minimizing $\KL(\rho_1\|\rho_\theta)=-\E_{s\sim\rho_1}[\log\rho_\theta(s)]+\mathrm{cst}$. The limiting cost is the Jacobian determinant, which constrains the choice of the map. Hamiltonian dynamics removes this cost structurally, since its Jacobian is one. Neural Hamiltonian Flows (NHF) instantiate this idea on $\R^d$ \citep{toth2020}. For sampling, one draws a pair $(q^0,p^0)\in\R^d\times\R^d$ from a phase-space base distribution, evolves it with a symplectic leapfrog integrator $\Phi_\theta$ for a learned Hamiltonian $H_\theta(q,p)=V_\theta(q)+K_\theta(p)$, and keeps only the final position. For training, the leapfrog map has unit Jacobian, so~\eqref{eq:nf} involves no determinant. However, the density of the final position is an intractable marginal over the momenta, $\rho_\theta(q)=\int f_\theta(q,p)\,\dd p$, so training is variational. An encoder proposes a momentum at each data point $q$ following $q_\phi(\cdot\mid q)$, the inverse map gives $(q^0,p^0)=\Phi_\theta^{-1}(q,p)$, and one maximizes the evidence lower bound (ELBO)
\begin{equation}
  \mathcal L(q;\theta,\phi,\sigma_p)
  =\mathbb E_{p\sim q_\phi(\cdot\mid q)}\Bigl[\log\rho_0(q^0)+\log\mathcal{N}\bigl(p^0;0,\sigma_p^2I\bigr)-\log q_\phi(p\mid q)\Bigr].
  \label{eq:elbo0}
\end{equation}
The fixed-kinetic variant freezes the kinetic energy to $\tfrac12\lVert p\rVert^2$ and learns only the potential $V_\theta$ \citep{souveton2024}. This reduces the number of parameters and makes the learned object interpretable: an energy landscape on the data space, whose wells attract the probability mass. This property is explained in Section~\ref{sec:thermo}. Both formulations, however, are Euclidean.

\subsection{The RNHF model}
\label{sec:rnhf}

RNHF keeps the fixed-kinetic principle and replaces every Euclidean ingredient by its Riemannian counterpart. The architecture is illustrated in Figure~\ref{fig:model}, and we now describe each block.

\begin{figure}[!t]
\centering
\resizebox{0.95\linewidth}{!}{%
\begin{tikzpicture}[
  box/.style={draw, rounded corners, align=center, inner sep=6pt, font=\small},
  inner/.style={draw, rounded corners=2pt, align=center, inner sep=4pt, font=\footnotesize, fill=gray!10},
  lbl/.style={font=\footnotesize\itshape},
  arr/.style={-{Latex[length=2.5mm]}, thick}
]
\node[box] (base) {base on $T^*\M$\\[2pt] $q^0 \sim \rho_0$\\ $p^0 \sim \Ng\big(0,\, \sigma_p^2 I\big)$};
\node[inner, right=20mm of base] (kick1) {half-kick\\ $p \leftarrow p - \tfrac{h}{2}\, \dd V_\theta(q)$};
\node[inner, right=5mm of kick1] (drift) {geodesic drift\\ $q \leftarrow \exp_q\!\big(h\, g_q^{-1} p\big)$\\ $p \leftarrow \PT{q}{q'}\, p$};
\node[inner, right=5mm of drift] (kick2) {half-kick};
\node[draw, rounded corners, fit=(kick1)(drift)(kick2), inner sep=8pt,
      label={[font=\small]above:{geodesic integrator~\eqref{eq:kdk}: $L$ steps, unit Jacobian on $T^*\M$}}] (lf) {};
\node[box, right=20mm of lf] (data) {target and encoder on $T^*\M$\\[2pt] data $q \sim \rho_1$\\ $p \sim \Ng\big(\mu_\phi(q),\, \Sigma_\phi(q)\big)$};
\draw[arr] ([yshift=5mm]base.east) -- ([yshift=5mm]lf.west);
\draw[arr] ([yshift=5mm]lf.east) -- ([yshift=5mm]data.west);
\node[lbl, above] at ($([yshift=5mm]base.east)!0.5!([yshift=5mm]lf.west)$) {sampling};
\draw[arr] ([yshift=-5mm]data.west) -- ([yshift=-5mm]lf.east);
\draw[arr] ([yshift=-5mm]lf.west) -- ([yshift=-5mm]base.east);
\node[lbl, below] at ($([yshift=-5mm]lf.east)!0.5!([yshift=-5mm]data.west)$) {training};
\end{tikzpicture}%
}
\caption{The RNHF model. Left to right (sampling): a phase-space point is drawn from the base and pushed through $L$ geodesic integrator steps, and the final position is the sample. Right to left (training): the encoder proposes a momentum at the data point, the discrete substeps are applied in reverse order, and the ELBO~\eqref{eq:elbo} evaluates the base densities at the resulting phase-space point.}
\label{fig:model}
\end{figure}

\paragraph{Base and encoder distributions.}
The position base $\rho_0$ can be a Gaussian on $\R^2$, a wrapped Gaussian on $\Hyp^2$, or the uniform law on $\Sph^2$. The momentum base is the isotropic distribution~\eqref{eq:maxwellian}, $p^0\mid q^0\sim\Ng(0,\sigma_p^2I)$, where $\sigma_p$ is a learned global scalar. The encoder is the intrinsic Gaussian $q_\phi(p\mid q)=\Ng\bigl(\mu_\phi(q),\Sigma_\phi(q)\bigr)$, whose mean is a cotangent vector and whose covariance is positive definite in an orthonormal basis.

\paragraph{Hamiltonian.}
The potential $V_\theta:\M\to\R$ takes a representation of the point (the coordinates themselves on $\R^2$, the ambient coordinates of the unit sphere in $\R^3$ on $\Sph^2$, and the coordinates of the Poincar\'e disk on the hyperboloid model of $\Hyp^2$). For reasons detailed in Section~\ref{sec:thermo}, the potential is parameterized as
\begin{equation}
  V_\theta=-\sigma_p^2\log U_\theta,
  \label{eq:param}
\end{equation}
where $U_\theta>0$ is the output of a neural network. The kinetic energy is fixed to $\tfrac12\,p^\top G(q)^{-1}p$.

\paragraph{Numerical integration.}
The flow is computed with $L$ steps of~\eqref{eq:kdk} up to the time $T=Lh$, and we denote by $\Phi_\theta$ the resulting map. Consecutive half-kicks can be fused, so a full pass costs $L+1$ evaluations of $\dd V_\theta$. The trajectory stays on the manifold. During training, its inverse is obtained by reversing the order of the substeps and reversing the timestep. The discrete map preserves the Liouville volume exactly but, at finite step size, does not conserve the Hamiltonian exactly.

\paragraph{Training and sampling.}
The principle is the same as in the Euclidean version. For a training point $q\in\M$, we draw $p\sim q_\phi(\cdot\mid q)$ and apply the inverse map to obtain $(q^0,p^0)=\Phi_\theta^{-1}(q,p)$. The evidence lower bound is written for densities on the manifold,
\begin{equation}
  \mathcal L(q;\theta,\phi,\sigma_p)
  =\mathbb E_{p\sim q_\phi(\cdot\mid q)}\Bigl[\log\rho_0(q^0)+\log\Ng\bigl(p^0;0,\sigma_p^2I\bigr)-\log q_\phi(p\mid q)\Bigr].
  \label{eq:elbo}
\end{equation}
Since $\Phi_\theta$ preserves volume, $\mathcal L(q)$ equals the log-density of the model at $q$ minus the KL divergence between the encoder and the true posterior of the momentum, so it is a lower bound on the log-likelihood that involves neither a Jacobian determinant nor a Riemannian divergence. Sampling is direct: draw $(q^0,p^0)$, apply $\Phi_\theta$, and keep the final position.

\paragraph{Kernel description.}
As a latent-variable model, RNHF can also be described by a Markov kernel, which spreads each starting point into a cloud of landing positions. For a starting position $q$, draw a momentum $P\sim\Ng(0,\sigma_p^2I)$, evolve $(q,P)$ under the Hamiltonian flow for a time $T$, and keep only the final position. Denoting by $\pr:T^*\M\to\M$ the canonical projection on positions, this defines the kernel
\begin{equation}
  \Kk(q,A)=\mathbb P\bigl(\pr\Flow_T(q,P)\in A\bigr),\qquad A\subseteq\M\ \text{Borel},
  \label{eq:kernel}
\end{equation}
and the position law after one pass is $\rho_T=\rho_0\Kk$. This description is the starting point of Section~\ref{sec:thermo}.

\section{Interpretability of the learned potential}
\label{sec:thermo}

This section explains what the learned potential means (Section~\ref{sec:dictionary}), what one pass of the flow produces and when the learned potential is the matched one (Section~\ref{sec:onepass}), and why it is worth approaching this canonical potential at all (Section~\ref{sec:hmc}). The analysis uses the exact flow $\Flow_T$, the integrator only adding an energy error that vanishes as $h \to 0$. We assume that $\M$ is complete, that $V$ is smooth and bounded below, and that $\rho_0$ and $\rho_1$ are smooth and positive. 

\subsection{A thermodynamic reading}
\label{sec:dictionary}

In statistical mechanics, a system with energy $H$ in contact with a heat bath at temperature $k_BT$ is found in phase space with a probability proportional to $e^{-H/k_BT}$. For the Hamiltonian~\eqref{eq:H}, this Boltzmann distribution factorizes into the Gaussian law~\eqref{eq:maxwellian} of the momenta, with variance $\tau=k_BT$, and a law of positions proportional to $e^{-V/k_BT}$. This suggests a simple dictionary. The momentum variance $\tau=\sigma_p^2$ plays the role of a temperature, and every potential $V$ comes with its Boltzmann distribution in position,
\begin{equation}
  \Pv=\frac{e^{-V/\tau}}{Z_{V,\tau}},\qquad Z_{V,\tau}=\int_\M e^{-V/\tau}\dvol,
  \label{eq:boltzmann}
\end{equation}
when $Z_{V,\tau}$ is finite. With the parameterization~\eqref{eq:param}, $\pi_{V_\theta,\tau}=U_\theta/\int_\M U_\theta\dvol$: the network output is, up to normalization, the Boltzmann distribution of the learned potential.

Conversely, we call matched potential the one that admits the data as its Boltzmann distribution. In thermodynamics, this is the free energy $-k_BT\log\rho_1$ up to a constant \citep{kirkwood1935}. Its wells are the most likely states, their relative depths give their populations, and the barriers between them control transitions. In our model, the temperature is $\tau$, and the matched potential is
\begin{equation}
  V_\star=-\tau\log\rho_1 .
  \label{eq:matched}
\end{equation}
A change of units shows how the momentum scale and the potential share their roles. Writing $p=\sigma_p\eta$ and $s=\sigma_pt$, the covariant Newton equation becomes
\begin{equation}
  \nabla_s q'=-\gradg\Bigl(\frac{V}{\tau}\Bigr),\qquad q'=\frac{\dd q}{\dd s},\qquad \eta\sim\Ng(0,I),
  \label{eq:reduced}
\end{equation}
integrated over the duration $\sigma_pT$. A pass therefore depends on the potential only through the dimensionless landscape $V/\tau$, and on the momentum scale only through the clock $\sigma_pT$. The parameterization~\eqref{eq:param} implements exactly this split: the network learns the dimensionless landscape $-\log U_\theta$, to be compared with $-\log\rho_1$, while $\sigma_p$ sets the clock. We say that the learned potential is the matched one when $V_\theta=V_\star$ up to a constant, that is, when $U_\theta\propto\rho_1$. 

The matched potential also has a simple mechanical meaning, which will guide the choice of the base. Think of the lifted base $\rho_0M_\tau$ as a gas of independent particles, and follow its density $\rho_t$ and its mean velocity $u_t$ in position space. Hamilton's equations imply two exact balance laws, known in stellar dynamics as the Jeans equations \citep{binney2008} and derived in Appendix~\ref{app:moments},
\begin{equation}
  \partial_t\rho_t+\divg(\rho_tu_t)=0,
  \qquad
  \partial_t(\rho_tu_t)+\divg S_t=-\rho_t\gradg V,
  \label{eq:jeans}
\end{equation}
where $S_t$ is the second moment of the velocities, or momentum flux. The first equation states that mass is conserved. The second is Newton's law for the gas: the momentum of a small region changes because particles carry momentum in and out of it, and because the potential exerts a force on them. When the velocities are isotropic around their mean, with variance $\tau$, the momentum flux becomes $S_t=\rho_t\,u_t\otimes u_t+\tau\rho_t\,g^{-1}$, and the second equation reduces to the Euler equation of an ideal gas at temperature $\tau$, whose pressure is $P_t=\tau\rho_t$:
\begin{equation*}
  \partial_tu_t+\nabla_{u_t}u_t=\underbrace{-\gradg V}_{\text{force}}\;\underbrace{-\,\tau\gradg\log\rho_t}_{\text{pressure}} .
\end{equation*}
The gas is then driven by two competing forces: the force of the potential, which pulls it into the wells, and the pressure force, which pushes it out of the regions where it is concentrated. A gas at rest is in equilibrium when the two forces balance, $\tau\gradg\rho=-\rho\gradg V$, that is, when $\rho\propto e^{-V/\tau}$ is the Boltzmann distribution. The matched potential is therefore the potential whose force exactly balances the pressure of the data: it holds the data at rest, in hydrostatic equilibrium. The isotropy assumption holds exactly for the lifted base at $t=0$, whose momenta are isotropic with variance $\tau$ and whose mean velocity vanishes, so that the gas starts to move with the acceleration
\begin{equation}
  \partial_tu_t\big|_{t=0}
  =\underbrace{-\gradg V}_{\text{force}}\;\underbrace{-\,\tau\gradg\log\rho_0}_{\text{pressure}}
  =\tau\,\gradg\log\frac{\Pv}{\rho_0}.
  \label{eq:acceleration}
\end{equation}
At later times, the velocities are in general no longer isotropic, the Euler equation is only an approximation, and the outcome of the competition between force and pressure is described exactly by the result of the next subsection.

\subsection{What one pass produces}
\label{sec:onepass}

Recall from Section~\ref{sec:rnhf} that one pass of the model is a kernel $\Kk$ that sends each starting position $q$ to a cloud $\Kk(q,\cdot)$ of arrival positions, so that the output of the pass is $\rho_T=\rho_0\Kk$. The following identity tells us what this output looks like.

\begin{proposition}[One pass is a Boltzmann distribution times a memory factor]\label{prop:decomposition}
Assume $Z_{V,\tau}<\infty$ and let $w=\rho_0/\Pv$. For every $T\ge0$,
\begin{equation}
  \rho_T=\Pv\,R_T,
  \qquad
  R_T(q)=\E\bigl[w(Q_T)\,\big|\,Q_0=q\bigr],
  \label{eq:decomposition}
\end{equation}
where $Q_T$ is the arrival position of a particle that starts at $q$ with a momentum drawn from $\Ng(0,\tau I)$. Consequently, up to an additive constant,
\begin{equation}
  V-V_\star=\underbrace{\tau\log R_T}_{\text{memory}}\;-\;\underbrace{\tau\log\frac{\rho_T}{\rho_1}}_{\text{model error}}.
  \label{eq:gap}
\end{equation}
\end{proposition}

The proof in Appendix~\ref{app:decomposition} uses three fundamental properties of Hamiltonian dynamics: volume preservation, energy conservation and time reversibility. The identity says that the output of one pass is the Boltzmann distribution of the potential, reweighted by a memory factor $R_T$. Equation~\eqref{eq:gap} then splits the gap between the learned potential and the matched one into a model error, which training reduces, and a memory term. 

\begin{theorem}[One-pass thermalization]\label{thm:thermalization}
Assume $Z_{V,\tau}<\infty$, and let $Q_0\sim\rho_0$ and $Q_T$ be the input and output positions of one pass. If $Q_T$ is independent of $Q_0$, then $\rho_T=\Pv$. If, in addition, the model is exact, $\rho_T=\rho_1$, then $V=V_\star$ up to a constant.
\end{theorem}

The proof can be found in Appendix~\ref{app:certificate}. This result states that when the model is exact and when the cloud does not depend on $q$, that is, when the pass forgets its starting point, the learned potential is the matched one. To illustrate the phenomenon, let us consider a Gaussian target $\N(m,a^2I)$ for which the matched potential is quadratic and the dynamics is a harmonic oscillator. In the normalized coordinates $(q-m)/a$ and $p/\sigma_p$, phase space simply rotates, and after a quarter turn, reached at the clock $\sigma_pT=\pi a/2$, the final position is the initial momentum, rescaled, while the initial position has moved into the final momentum \citep{holderrieth2024}. The output is then exactly the data and carries no information about the starting point. This reading has two limits, however: the matched potential may not always be the only one that samples the target and it may not always be an exact solution. The ELBO only asks the output to match the data but in practice \citet{souveton2024} observed on Euclidean spaces that, with a nearly flat position base broad enough to cover the data, the learned potential places its wells at the modes of the data. The fluid picture agrees with this at the start of the pass: by~\eqref{eq:acceleration}, a nearly flat base feels no pressure at $t=0$ and is set in motion by the force alone. Equation~\eqref{eq:acceleration} only describes the first instants, however, and the choice among the exact passes is made by the optimizer, not by the objective. 

\subsection{Why interpretability matters: recycling samples into an exact sampler}
\label{sec:hmc}

We now show that learning an interpretable potential has a direct computational use which tolerates the mismatch between the learned and the data landscape. Consider a Boltzmann distribution $\pi\propto e^{-E}$ on $\M$ at temperature one, whose energy $E$ can be evaluated but at a high cost, or without access to its gradient, as for a black-box simulation code. Some samples of $\pi$ are already available, for instance from a previous simulation or from a short run of a sampler, and more exact samples are needed. Hamiltonian Monte Carlo samples $\pi$ exactly \citep{duane1987, neal2011}, on manifolds with the same geodesic integrator as RNHF \citep{byrne2013}, but each transition integrates a trajectory that requires about $L$ evaluations of the gradient of $E$. Surrogate Hamiltonian Monte Carlo replaces $E$ by a cheap approximation $\widehat E$ during the integration, and uses $E$ only in the Metropolis test at the end of the trajectory \citep{rasmussen2003, zhang2017, li2019}. One transition from the current position $q$ reads as follows: draw a momentum $p\sim\Ng(0,I)$, compute $(q',p')$ with $L$ steps of the integrator~\eqref{eq:kdk} for the potential $\widehat E$ over a duration $\ell$, and move to $q'$ with probability $\min\{1,e^{-\Delta H}\}$, where
\begin{equation}
  \Delta H=E(q')+\tfrac12\normg{p'}^2-E(q)-\tfrac12\normg{p}^2 .
  \label{eq:dH}
\end{equation}

For every smooth surrogate $\widehat E$, every duration and every number of steps, this transition leaves $\pi$ invariant for a volume-preserving and reversible integrator. Writing $\Delta\widehat H$ for the energy error of the integrator for the surrogate Hamiltonian $\widehat E+\tfrac12\normg{p}^2$, and noting that the kinetic energy changes by the same amount in $\Delta H$ and in $\Delta\widehat H$, we obtain
\begin{equation}
  \Delta H=\bigl[E-\widehat E\bigr](q')-\bigl[E-\widehat E\bigr](q)+\Delta\widehat H,
  \label{eq:dH-split}
\end{equation}
which says how the surrogate affects the efficiency. Since the integrator nearly conserves the surrogate energy, a proposal is accepted with high probability when $E-\widehat E$ takes nearly the same value at both ends of the trajectory. We then simply use the learned landscape as the surrogate,
\begin{equation}
  \widehat E=-\log U_\theta .
  \label{eq:surrogate}
\end{equation}
Remaining errors in the shape of a well lower the acceptance rate, but never bias the samples. A coarse approximation of the matched potential is thus all the sampler needs. The duration of the trajectories must also be chosen, and the quarter turn of Section~\ref{sec:onepass} tells us how. In a quadratic well of curvature $k$, a trajectory of duration $\pi/(2\sqrt k)$ turns phase space by a quarter turn and produces a position independent of the starting one, which is the best possible move. We use the average curvature of the surrogate where the samples lie,
\begin{equation}
  \ell=\frac\pi2\sqrt{\frac{d}{\E_{\rho_1}[\Delta_g\widehat E]}},
  \qquad
  \Delta_g\widehat E(x)=\sum_{i=1}^d\frac{\dd^2}{\dd t^2}\,\widehat E\bigl(\exp_x(te_i)\bigr)\Big|_{t=0},
  \label{eq:ell}
\end{equation}
where $\Delta_g$ is the Laplace--Beltrami operator, which extends the Laplacian to curved spaces, $e_1,\dots,e_d$ is an orthonormal basis of $T_x\M$, and the expectation is estimated on the available samples by automatic differentiation of the network.

\section{Numerical experiments}
\label{sec:experiments}

We evaluate RNHF on three controlled two-dimensional model spaces: the Euclidean plane $\mathbb R^2$, the hyperbolic plane $\mathbb H^2$, and the sphere $\mathbb S^2$. On each space, we compare with a Riemannian continuous normalizing flow (RCNF) \citep{mathieu2020} built from the same base distribution and the same geometric tools. We demonstrate that RNHF can recover multimodal targets on different curvatures with sampling quality and cost competitive with RCNF (Section~\ref{sec:sampling-results}), we show that the learned potentials close to the matched ones (Section~\ref{sec:interpretability}), and we test the use of the learned landscape as a surrogate energy on astrophysical data (Section~\ref{sec:hmc-exp}). Details of the RCNF implementation and of the experimental settings are given in Appendices~\ref{app:rcnf}--\ref{app:icecube}. The code is available at \url{https://github.com/vincentsouveton/RiemannianNHF}.

\subsection{Sampling performance and computational cost}
\label{sec:sampling-results}

Figure~\ref{fig:samples} shows the kernel density estimations (KDE) of 2,000 generated points for each model. Both methods recover the two modes on all three geometries with the correct weights: 0.6 for the left mode, 0.4 for the right mode. 
\begin{figure}[!ht]
    \centering
    \includegraphics[width=0.75\textwidth]{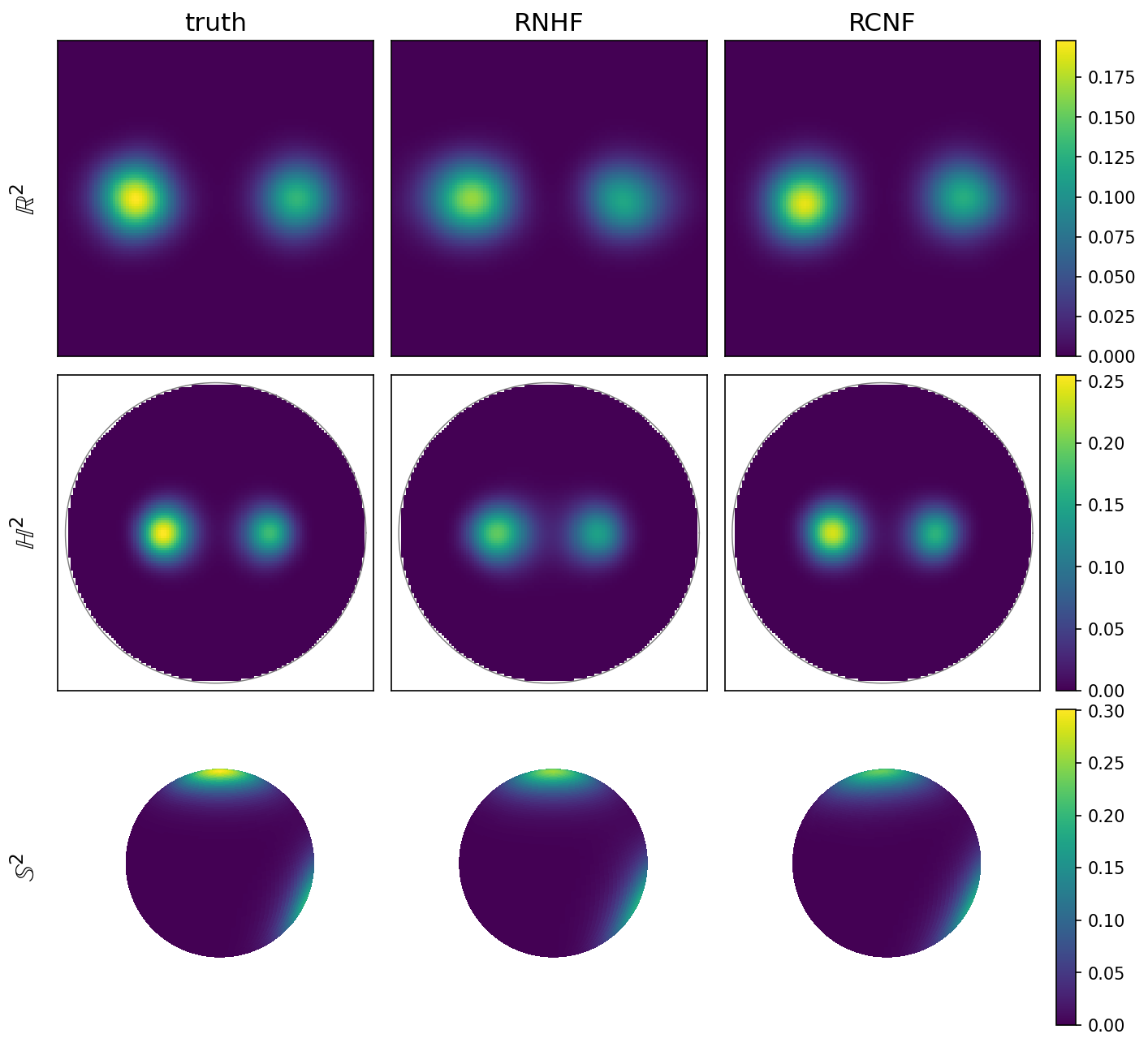}
    \caption{Ground-truth and generated KDE of the samples on the three controlled model spaces. Rows correspond to $\mathbb R^2$, $\mathbb H^2$ (visualized in the Poincare disk), and $\mathbb S^2$. Columns show the target, RNHF, and RCNF.}
    \label{fig:samples}
\end{figure}

Table~\ref{tab:results} gives a more nuanced picture. On $\mathbb{R}^2$, RCNF is better on four of the six criteria, RNHF taking the mean $\ell_2$ and the recall. On $\mathbb{H}^2$, RCNF leads on most metrics, and the largest RNHF discrepancy is the energy distance, which is sensitive to a small number of far-from-origin samples in a geometry whose volume expands exponentially with radius. On $\mathbb{S}^2$, where no such volume growth occurs, RNHF is ahead on five of the six criteria, RCNF being slightly better only on recall. Overall these results support competitiveness on the controlled tasks, and a clear RNHF advantage in positive curvature, without establishing a uniform ordering across geometries.
\begin{table}[t]
\centering
\caption{Sample-based evaluation on the three model spaces, in the chart-free space. Arrows indicate the preferred direction and the better of the two models is shown in bold. Average over three independent runs.}
\label{tab:results}
\begin{tabular}{l cc cc cc}
\toprule
 & \multicolumn{2}{c}{$\mathbb{R}^2$} & \multicolumn{2}{c}{$\mathbb{H}^2$} & \multicolumn{2}{c}{$\mathbb{S}^2$} \\
\cmidrule(lr){2-3}\cmidrule(lr){4-5}\cmidrule(lr){6-7}
Metric & RNHF & RCNF & RNHF & RCNF & RNHF & RCNF \\
\midrule
$\widehat{\mathrm{KL}}(\rho_1\|p_\theta)$ $\downarrow$ & 0.0554 & \textbf{0.0290} & 0.0961 & \textbf{0.0495} & \textbf{0.0385} & 0.0920 \\
Mean $\ell_2$ $\downarrow$ & \textbf{0.0460} & 0.0613 & 0.0372 & \textbf{0.0252} & \textbf{0.0449} & 0.1098 \\
Cov.\ Frobenius $\downarrow$ & 0.1192 & \textbf{0.1062} & 0.1072 & \textbf{0.0624} & \textbf{0.0284} & 0.0648 \\
Energy dist.\ $\downarrow$ & 0.0035 & \textbf{0.0031} & 0.0052 & \textbf{0.0016} & \textbf{0.0042} & 0.0188 \\
Precision $\uparrow$ & 0.8688 & \textbf{0.8937} & 0.8112 & \textbf{0.8602} & \textbf{0.8402} & 0.8123 \\
Recall $\uparrow$ & \textbf{0.9053} & 0.8988 & \textbf{0.9057} & 0.8998 & 0.9037 & \textbf{0.9195} \\
\bottomrule
\end{tabular}
\end{table}

The computational difference follows directly from the two objectives. RNHF differentiates a scalar potential to obtain a single force vector, and preserves phase-space volume structurally, so its likelihood carries no divergence term. RCNF instead differentiates a vector field and evaluates the trace of its Jacobian. Computed exactly, this trace costs one directional derivative per frame direction, i.e.\ $d$ backward passes per integration step. Sampling reverses the balance: RNHF carries both position and momentum and applies KDK steps, whereas RCNF integrates a single field. Under a matched protocol, a training step is $2.2\times$ faster for RNHF, a decisive advantage, since training dominates the overall computational budget. At inference the ordering reverses and sampling is about $1.3\times$ slower for RNHF, but the absolute cost stays below a millisecond per generated point on every geometry, making it essentially free. This last point also distinguishes RNHF from diffusion models, whose generation requires simulating a long stochastic chain.

\subsection{Interpretability of the learned dynamics}
\label{sec:interpretability}

RNHF parameterizes $V_\theta=-\sigma_p^2\log U_\theta$, with $U_\theta>0$. Therefore its Boltzmann distribution is $\pi_{V_\theta,\tau}(q)
  =\frac{U_\theta(q)}{\int_{\M}U_\theta\dvol}$. We begin with a unimodal 1D Gaussian target, for which the quarter turn of Section~\ref{sec:onepass} is exact. We train a RNHF on $\rho_1=\mathcal N(0,a^2)$ with $a=0.7$, starting from a broad Gaussian position base with $\sigma_b=1.5$. The learned momentum scale converges to $\sigma_p=1.10$ against the predicted $\sigma_p^\star=\frac{\pi}{2}a=1.10$. Figure~\ref{fig:learned_lens} shows that the phase-space mechanism itself survives learning. The colour gradient, which tracks the initial position, rotates from horizontal to vertical as the flow proceeds. All the information carried by the base has moved into momentum, and the position marginal is the target, in exact accordance with the quarter turn of Section~\ref{sec:onepass}, and the pass is memoryless in the sense of Theorem~\ref{thm:thermalization}.

\begin{figure}[!ht]
  \centering
  \includegraphics[width=0.75\textwidth]{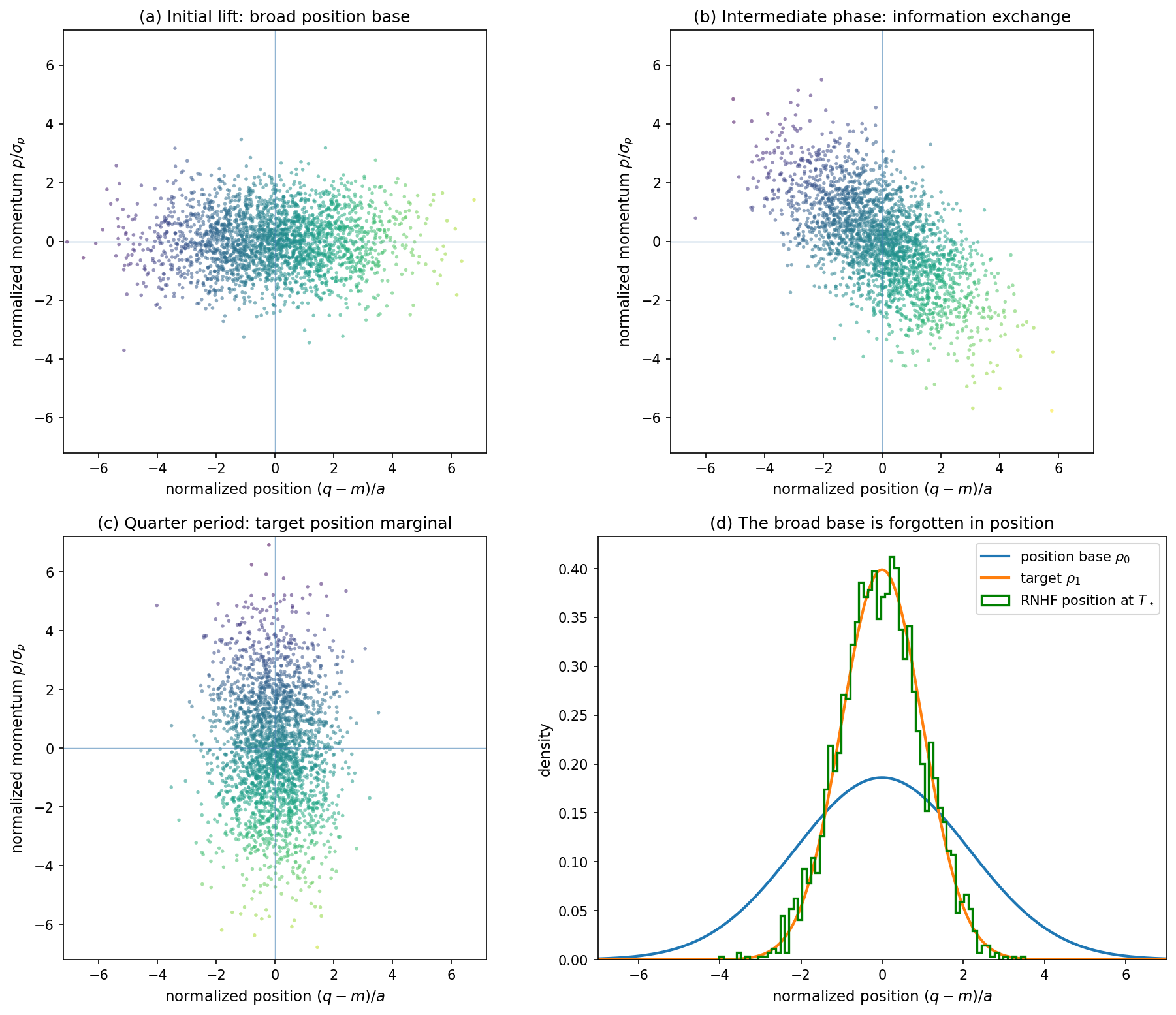}
  \caption{Learned finite-time phase-space rotation for a one-dimensional Gaussian target, shown in normalized phase-space coordinates. Colors encode the initial position. (a) Initially, the phase-space cloud is elongated along position. (b) The learned flow rotates the cloud and creates position--momentum correlation. (c) At a quarter period, the position coordinate has the target law, while the base information is now stored in momentum. (d) The final model position marginal matches the target.}
  \label{fig:learned_lens}
\end{figure}

Next, Figure~\ref{fig:potentials} shows the learned potential on each model space. In all three cases, the two wells align with the target modes. We quantify the agreement on the evaluation domain used for each plot. The first score is the Pearson correlation between $\log U_\theta$ and $\log\rho_1$ after centering, which removes the arbitrary additive constant in the potential. The second is the slope $\alpha$ in the regression $\log U_\theta=\alpha\log\rho_1+\beta$. If the learned potential is the matched one, then both correlation and slope should be one. The observed shape correlations are $0.96$ on $\mathbb{R}^2$, $0.91$ on $\mathbb{H}^2$, and $0.96$ on $\mathbb{S}^2$, with slopes $0.89$, $0.83$, and $0.98$ respectively. The learned landscapes therefore recover the matched potential closely in both shape and amplitude, with $\mathbb{H}^2$ slightly worse than the other two geometries. This gap reflects a deliberate trade-off in the choice of position base. As discussed in Section~\ref{sec:onepass}, a broad base improves the agreement with the matched potential. However, on $\mathbb{H}^2$, a broad base also places mass far from the origin, where the exponential volume growth turns a handful of samples into outliers that degrade the metrics. We therefore use a moderate base variance, trading a small amount of interpretability for cleaner samples. To be more specific, we measured that raising the position base standard deviation from 0.8 to 1.0 lifts the slope to 0.94. 

\begin{figure}[!ht]
    \centering
    \includegraphics[width=0.85\textwidth]{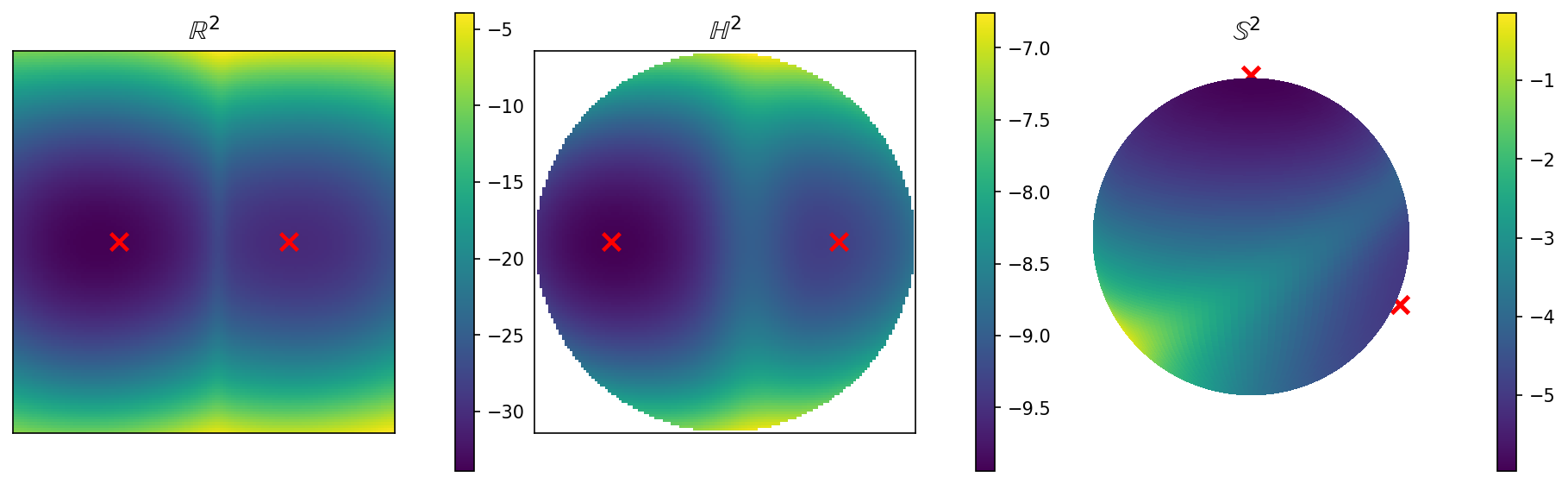}
    \caption{Potentials learned by RNHF on $\mathbb R^2$, $\mathbb H^2$, and $\mathbb S^2$. Red crosses mark the target mode centers. In each geometry the wells of the learned potential align with the modes, and the deeper well corresponds to the higher-weight mode ($0.6$ vs $0.4$).}
    \label{fig:potentials}
\end{figure}

Figure~\ref{fig:refocusing} illustrates the role of the momentum scale $\sigma_p$ using the matched potential $V_\star=-\sigma_p^2\log\rho_1$. Starting from the position base used in the experiments, we draw a Gaussian momentum of scale $\sigma_p$, integrate the Hamiltonian flow for $T=1$, and measure the KL divergence between the target and the resulting distribution. When $\sigma_p$ is too small, the dynamics is too slow, whereas when it is too large, the trajectories overshoot and disperse, so the distance to the target is minimised at an intermediate value. On the three spaces, the KL divergence exhibits an oscillatory behavior that is consistent with the phase-space rotation of Section~\ref{sec:onepass}. The key observation is that the optimum of each curve coincides with the momentum scale that our RNHF models learn during training, even though nothing in the loss prescribes it explicitly.

\begin{figure}[!ht]
    \centering
    \includegraphics[width=0.85\textwidth]{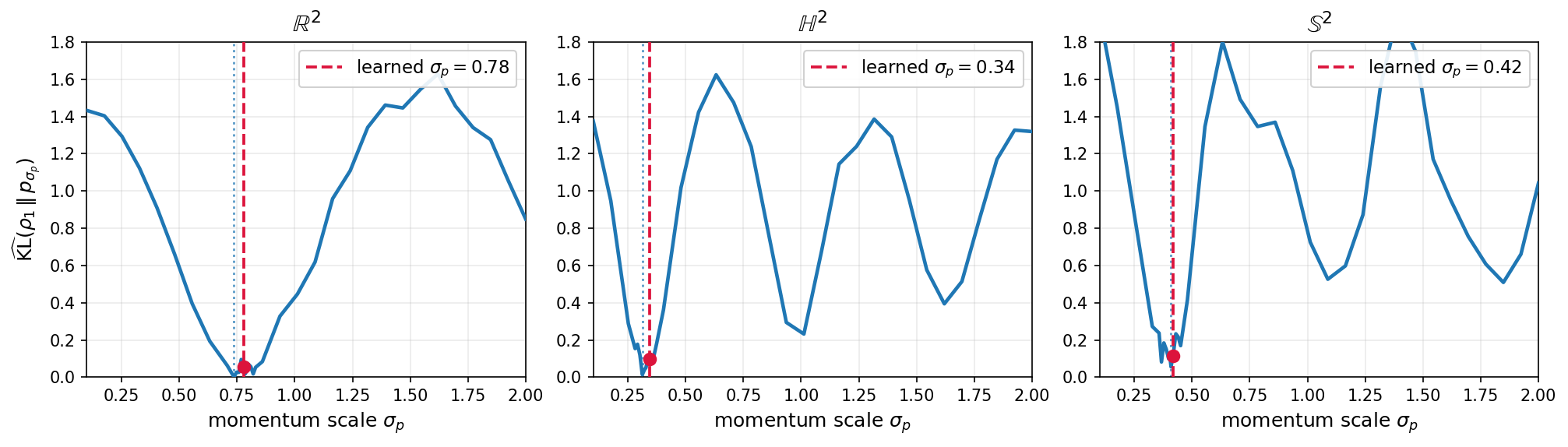}
    \caption{Transport with the matched potential $V_\star=-\sigma_p^2\log\rho_1$. Starting from the position base, we integrate the Hamiltonian flow for $T=1$ and estimate the KL divergence between the target and the model distribution, as a function of the momentum scale $\sigma_p$. Each curve has a clear optimum and the momentum scale learned by our trained RNHF models (red dashed line) is close to it on all three geometries.}
    \label{fig:refocusing}
\end{figure}

\subsection{An astrophysics test: localizing a neutrino source}
\label{sec:hmc-exp}
 
We now test the surrogate sampler of Section~\ref{sec:hmc} on a problem where the energy is genuinely expensive: locating on the sky a source of high-energy neutrinos. The IceCube detector, buried in the Antarctic ice, records the arrival directions of neutrinos. Most of them are produced in the atmosphere of the Earth and arrive from all over the sky, but a few come from astrophysical sources, which appear as a small excess of events around one direction. Finding this direction is a Bayesian inference problem on the sphere $\Sph^2$. The data are the public IceCube sample of $1{,}134{,}450$ events recorded between 2008 and 2018 \citep{icecube2021}. Each event $i$ comes with its arrival direction $x_i\in\Sph^2$, reconstructed with an angular uncertainty $\sigma_i$. The unknown is the direction $n\in\Sph^2$ of the source. Each event is either a signal event, emitted by the source, with probability $n_s/M$, where $M$ is the number of events and $n_s$ the number of signal events, or a background event. A signal event arrives near $n$, blurred by the reconstruction error: its direction follows the von Mises--Fisher density
\begin{equation*}
  \mathrm{vMF}(x;n,\kappa)=\frac{\kappa}{4\pi\sinh\kappa}\,e^{\kappa\,n\cdot x},
  \qquad \kappa=\sigma_i^{-2},
\end{equation*}
which is the analogue on the sphere of an isotropic Gaussian of width $\sigma_i$ centered at $n$. A background event follows an estimated density $B$ that only depends on the declination, the latitude on the sky, since the detector turns with the Earth and sees all longitudes in the same way. The events being independent, the likelihood of the data is \citep{braun2008}.
\begin{equation*}
  p(x_1,\dots,x_M\mid n)=\prod_{i=1}^M\Bigl[\frac{n_s}{M}\,\mathrm{vMF}\bigl(x_i;n,\sigma_i^{-2}\bigr)+\Bigl(1-\frac{n_s}{M}\Bigr)B(x_i)\Bigr].
\end{equation*}
We assume that $n_s$ is known and we use the directions only. The prior is taken uniform with respect to the area of the sphere on a cap $C$ of radius $1^\circ$ around the maximum $\widehat n$ of the likelihood: $p(n)\propto\mathbf 1_C(n)$. By Bayes' rule, the posterior distribution of the direction is $\pi(n)\propto p(n)\,p(x_1,\dots,x_M\mid n)\propto\mathbf 1_C(n)\,e^{-E(n)}$, with the energy
\begin{equation}
  E(n)=-\sum_{i=1}^{M}\log\Bigl[\frac{n_s}{M}\,\mathrm{vMF}\bigl(x_i;n,\sigma_i^{-2}\bigr)+\Bigl(1-\frac{n_s}{M}\Bigr)B(x_i)\Bigr].
  \label{eq:icecube}
\end{equation}
To know the right answer, we create the source ourselves, as IceCube does to validate its analyses. We first shuffle the longitudes of all events at random, which erases any real source while preserving the dependence of the detector on the declination. We then add $n_s=100$ signal events, drawn from the model above around the direction of the active galaxy NGC~1068 \citep{icecube2022}, with uncertainties $\sigma_i$ taken from events of similar declination. The posterior is then a narrow well, about $0.1^\circ$ wide, near the injected source. Since $E$ is a sum over more than a million events, one evaluation of $E$ or of its gradient takes about $18$~ms, against $0.6$~ms for a gradient of the network used by RNHF: this is exactly the situation of Section~\ref{sec:hmc}.
 
The protocol gives both methods the same starting point. Geodesic Hamiltonian Monte Carlo with the true gradients \citep{byrne2013}, with $L=20$ steps and the duration~\eqn{eq:ell} computed for $E$, is started from the maximum of likelihood $\widehat n$, and its first $1{,}000$ states after a burn-in are kept. These are the only data RNHF sees. Each method then produces $9{,}000$ further samples. Exact Hamiltonian Monte Carlo simply continues its chain. For our method, RNHF is trained on the $1{,}000$ initial states, and the surrogate sampler, with the surrogate~\eqn{eq:surrogate}, $L=20$ steps and the duration~\eqn{eq:ell}, starts from the $1{,}000$th state. Each method is run with three random seeds on the same data set, and a long run of exact Hamiltonian Monte Carlo, with $20{,}000$ states, serves as reference. Appendix~\ref{app:icecube} gives additional details.
 
\begin{figure}[!htb]
    \centering
    \includegraphics[width=\textwidth]{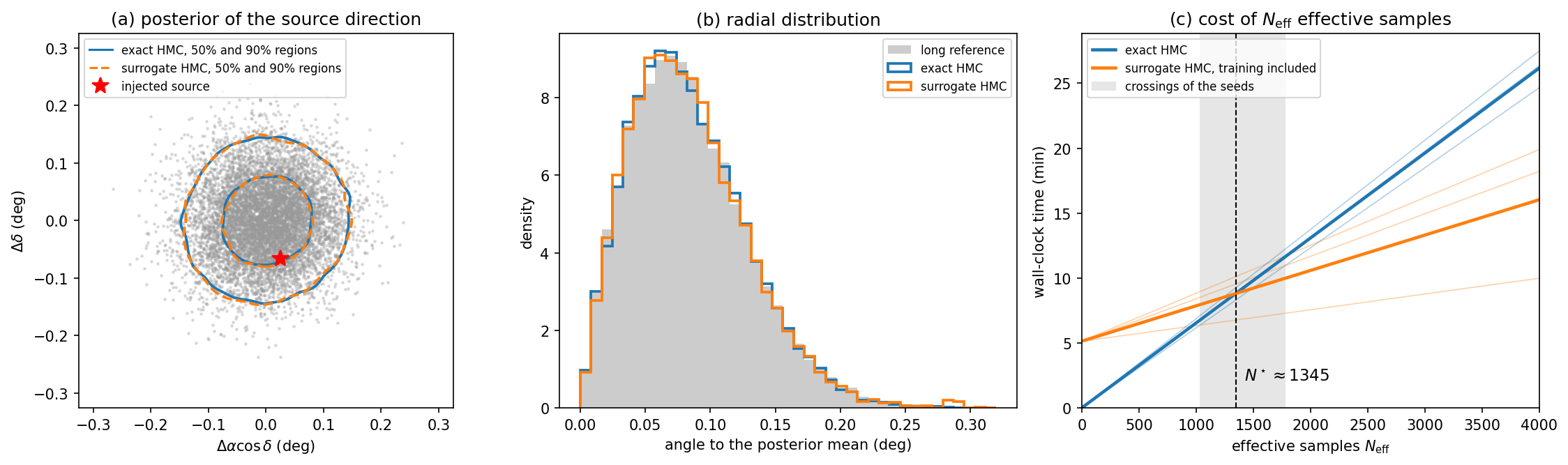}
    \caption{Localization of an injected neutrino source with exact and surrogate Hamiltonian Monte Carlo. (a) Regions containing $50\%$ and $90\%$ of the posterior mass, estimated from the $27{,}000$ second-phase samples of each method (three seeds), in a chart centered at the posterior mean. Grey dots are samples of exact Hamiltonian Monte Carlo, and the star is the injected source. (b) Distribution of the angular distance to the posterior mean, with the long reference run in grey. (c) Wall-clock cost of $N_{\mathrm{eff}}$ effective samples for both methods, including the training of RNHF for the surrogate sampler. Thick lines are mean costs, thin lines individual seeds, and the band is the range of the break-even points of the three seeds.}
    \label{fig:icecube}
\end{figure}
 
The first message of Figure~\ref{fig:icecube} is that the surrogate sampler is accurate. Figure~\ref{fig:icecube}a shows that the two methods give the same credible regions, and Figure~\ref{fig:icecube}b that the distribution of the distance to the posterior mean coincides with that of the long reference. Quantitatively, the angle between the posterior mean and the injected source is $0.0709^\circ\pm0.0011^\circ$ for the surrogate sampler and $0.0709^\circ\pm0.0005^\circ$ for exact Hamiltonian Monte Carlo, and the radius of the $90\%$ credible region is $0.1445^\circ$ and $0.1443^\circ$, against $0.1453^\circ$ for the reference. For every seed, the posterior mean, this angle and this radius given by the surrogate sampler differ from the reference by less than one standard error. The learned landscape is remarkably accurate where the samples lie: over the $1{,}000$ training states, $-\log U_\theta$ has a correlation of $0.986$ with the true energy $E$, and a regression slope of $1.05$. The wells have the right depth and $92\%$ of the surrogate proposals are accepted.
 
The second message is that the surrogate sampler is cheaper. In wall-clock time, the $9{,}000$ transitions take about $3{,}400$~s for exact Hamiltonian Monte Carlo and $275$~s for the surrogate sampler, twelve times less. What matters, however, is the cost of an effective sample, that is, the cost of a transition multiplied by the integrated autocorrelation time $\tau$ of the chain. Exact Hamiltonian Monte Carlo produces nearly independent samples, $\tau=1.0$, whereas the surrogate chain has $\tau=5.3$ on average, so that an effective sample costs $0.39$~s with exact Hamiltonian Monte Carlo and $0.16$~s with the surrogate sampler, a gain of $2.4$. The surrogate sampler must also pay once for the training of RNHF and the computation of $\ell$, about $T_0=308$~s. The cost of $N_{\mathrm{eff}}$ effective samples is therefore $c_HN_{\mathrm{eff}}$ for exact Hamiltonian Monte Carlo and $T_0+c_SN_{\mathrm{eff}}$ for the surrogate sampler, and the two lines of Figure~\ref{fig:icecube}c cross at
\begin{equation}
  N^\star=\frac{T_0}{c_H-c_S}\approx1{,}345 .
  \label{eq:breakeven}
\end{equation}
Beyond this number of effective samples, between $1{,}036$ and $1{,}773$ depending on the seed, the surrogate sampler is the faster method, and the speed-up approaches $c_H/c_S\approx2.4$ as $N_{\mathrm{eff}}$ grows. One shoud note that the surrogate chain mixes more slowly as it is dominated by rare episodes in the tail of the posterior, where the training states are few and the surrogate is less accurate: in two of the three seeds, the chain stays stuck for about $30$ consecutive transitions and the acceptance rate falls to $\sim 0.5$ on the $1\%$ highest energies, against $0.92$ overall. The third seed, without such an episode, has $\tau=2.4$ and a gain of $5.1$ per effective sample in wall-clock time.

\section{Conclusion}
\label{sec:conclusion}
 
We introduced Riemannian Neural Hamiltonian Flows, an intrinsic extension of fixed-kinetic Neural Hamiltonian Flows to Riemannian manifolds. Built from the kinetic energy of the metric, intrinsic Gaussian momenta, a learned potential and a geodesic integrator, the model stays on the manifold and preserves phase-space volume exactly, so that training needs neither Jacobian determinants nor divergences. On three geometries, it is competitive with a Riemannian continuous normalizing flow at a lower training cost. Reading the momentum variance as a temperature, we showed that one pass of the flow outputs the Boltzmann distribution of the learned potential multiplied by a memory factor, so that a pass that forgets its starting point learns the matched potential, whose force balances the pressure of the data. This reading has a computational application. When the energy of the data is known but expensive, the landscape learned from existing samples acts as a surrogate energy for Hamiltonian Monte Carlo. On the localization of a neutrino source from about a million events, the learned landscape had the right depth and the surrogate sampler agreed with exact Hamiltonian Monte Carlo within one standard error while becoming the faster method beyond about $1{,}300$ effective samples, training included. Its gain was limited by rare episodes in the tail of the posterior, which more training states in that region should reduce. Two questions remain open. On the theoretical side, which pairs $(\rho_0,\rho_1)$ can be connected, exactly or to a prescribed accuracy, by one pass of an autonomous potential with an isotropic, non-degenerate distribution of the initial momenta? On the numerical side, manifolds without closed-form geodesic flows will require implicit or approximate symplectic integrators, at a higher cost for training and inference. We believe that coupling gradient-based MCMC samplers with Hamiltonian generative models trained on their own samples can make exact inference affordable in complex computational pipelines whose energy and gradients are expensive to evaluate.

\section*{Use of generative AI statement}
We acknowledge the use of Claude Opus 5.5 in the preparation and refinement of the manuscript, the literature search, the improvement of Fig.~\ref{fig:geometry}, and the enhancement of the code documentation.

\bibliography{refs}

@article{albergo2025,
  author  = {Michael S. Albergo and Nicholas M. Boffi and Eric Vanden-Eijnden},
  title   = {Stochastic Interpolants: A Unifying Framework for Flows and Diffusions},
  journal = {Journal of Machine Learning Research},
  year    = {2025},
  volume  = {26},
  number  = {1},
  pages   = {1--61},
  url     = {http://jmlr.org/papers/volume26/23-1605/23-1605.pdf}
}

@article{betancourt2017,
  author    = {Michael Betancourt},
  title     = {A Conceptual Introduction to Hamiltonian Monte Carlo},
  journal   = {arXiv preprint arXiv:1701.02434},
  year      = {2017},
  url       = {https://arxiv.org/abs/1701.02434}
}

@article{braun2008,
  title={Methods for point source analysis in high energy neutrino telescopes},
  author={Braun, Jim and Dumm, Jon and De Palma, Francesco and Finley, Chad and Karle, Albrecht and Montaruli, Teresa},
  journal={Astroparticle Physics},
  volume={29},
  number={4},
  pages={299--305},
  year={2008}
}

@article{byrne2013,
  author  = {Byrne, Simon and Girolami, Mark},
  title   = {Geodesic {M}onte {C}arlo on Embedded Manifolds},
  journal = {Scandinavian Journal of Statistics},
  volume  = {40},
  number  = {4},
  pages   = {825--845},
  year    = {2013}
}

@inproceedings{caterini2018,
  author    = {Caterini, Anthony L. and Doucet, Arnaud and Sejdinovic, Dino},
  title     = {Hamiltonian Variational Auto-Encoder},
  booktitle = {Advances in Neural Information Processing Systems 31},
  year      = {2018}
}

@inproceedings{chen2024,
  author    = {Chen, Ricky T. Q. and Lipman, Yaron},
  title     = {Flow Matching on General Geometries},
  booktitle = {International Conference on Learning Representations},
  year      = {2024}
}

@inproceedings{chen2018,
  author    = {Chen, Ricky T. Q. and Rubanova, Yulia and Bettencourt, Jesse and Duvenaud, David},
  title     = {Neural Ordinary Differential Equations},
  booktitle = {Advances in Neural Information Processing Systems 31},
  year      = {2018}
}

@book{docarmo1992,
  author    = {do Carmo, Manfredo P.},
  title     = {Riemannian Geometry},
  publisher = {Birkh\"auser},
  year      = {1992}
}

@inproceedings{debortoli2022,
  author    = {De Bortoli, Valentin and Mathieu, Emile and Hutchinson, Michael and
               Thornton, James and Teh, Yee Whye and Doucet, Arnaud},
  title     = {Riemannian Score-Based Generative Modelling},
  booktitle = {Advances in Neural Information Processing Systems 35},
  year      = {2022}
}

@inproceedings{dinh2017,
  author    = {Dinh, Laurent and Sohl-Dickstein, Jascha and Bengio, Samy},
  title     = {Density Estimation Using {R}eal {NVP}},
  booktitle = {International Conference on Learning Representations},
  year      = {2017}
}

@ARTICLE{duane1987,
       author = {{Duane}, Simon and {Kennedy}, A.~D. and {Pendleton}, Brian J. and {Roweth}, Duncan},
        title = "{Hybrid Monte Carlo}",
      journal = {Physics Letters B},
         year = 1987,
        month = sep,
       volume = {195},
       number = {2},
        pages = {216-222},
          doi = {10.1016/0370-2693(87)91197-X},
       adsurl = {https://ui.adsabs.harvard.edu/abs/1987PhLB..195..216D}
}

@article{gemici2016,
  author  = {Gemici, Mevlana C. and Rezende, Danilo J. and Mohamed, Shakir},
  title   = {Normalizing Flows on {R}iemannian Manifolds},
  journal = {arXiv preprint arXiv:1611.02304},
  year    = {2016}
}

@article{girolami2011,
  author  = {Girolami, Mark and Calderhead, Ben},
  title   = {{R}iemann Manifold {L}angevin and {H}amiltonian {M}onte {C}arlo Methods},
  journal = {Journal of the Royal Statistical Society: Series B},
  volume  = {73},
  number  = {2},
  pages   = {123--214},
  year    = {2011}
}

@inproceedings{goodfellow2014,
  author    = {Goodfellow, Ian and Pouget-Abadie, Jean and Mirza, Mehdi and Xu, Bing and
               Warde-Farley, David and Ozair, Sherjil and Courville, Aaron and Bengio, Yoshua},
  title     = {Generative Adversarial Nets},
  booktitle = {Advances in Neural Information Processing Systems 27},
  year      = {2014}
}

@book{goodfellow2016,
  title={Deep Learning},
  author={Ian Goodfellow and Yoshua Bengio and Aaron Courville},
  publisher={MIT Press},
  year={2016},
  url={http://www.deeplearningbook.org}
}

@inproceedings{grathwohl2019,
  author    = {Grathwohl, Will and Chen, Ricky T. Q. and Bettencourt, Jesse and
               Sutskever, Ilya and Duvenaud, David},
  title     = {{FFJORD}: Free-Form Continuous Dynamics for Scalable Reversible Generative Models},
  booktitle = {International Conference on Learning Representations},
  year      = {2019}
}

@book{hairer2006,
  author    = {Hairer, Ernst and Lubich, Christian and Wanner, Gerhard},
  title     = {Geometric Numerical Integration},
  publisher = {Springer},
  edition   = {2nd},
  year      = {2006}
}

@inproceedings{ho2020,
  author    = {Ho, Jonathan and Jain, Ajay and Abbeel, Pieter},
  title     = {Denoising Diffusion Probabilistic Models},
  booktitle = {Advances in Neural Information Processing Systems 33},
  year      = {2020}
}

@inproceedings{holderrieth2024,
  author    = {Holderrieth, Peter and Xu, Yilun and Jaakkola, Tommi},
  title     = {Hamiltonian Score Matching and Generative Flows},
  booktitle = {Advances in Neural Information Processing Systems 37},
  year      = {2024}
}

@article{icecube2020prl,
  title   = {Time-integrated neutrino source searches with 10 years of {IceCube} data},
  author  = {{IceCube Collaboration}},
  journal = {Physical Review Letters},
  volume  = {124},
  number  = {5},
  pages   = {051103},
  year    = {2020},
  doi     = {10.1103/PhysRevLett.124.051103}
}

@misc{icecube2021,
  title={{IceCube} data for neutrino point-source searches: years 2008--2018},
  author={{IceCube Collaboration}},
  howpublished={arXiv:2101.09836},
  year={2021}
}

@misc{icecube2021data,
  title        = {{IceCube} data for neutrino point-source searches: years 2008--2018},
  author       = {{IceCube Collaboration}},
  howpublished = {Harvard Dataverse},
  year         = {2021},
  doi          = {10.7910/DVN/VKL316},
  url          = {https://doi.org/10.7910/DVN/VKL316}
}

@article{icecube2022,
  title={Evidence for neutrino emission from the nearby active galaxy {NGC} 1068},
  author={{IceCube Collaboration}},
  journal={Science},
  volume={378},
  number={6619},
  pages={538--543},
  year={2022}
}

@inproceedings{kingma2014,
  author    = {Kingma, Diederik P. and Welling, Max},
  title     = {Auto-Encoding Variational {B}ayes},
  booktitle = {International Conference on Learning Representations},
  year      = {2014}
}

@book{landau1976,
  author = {Landau, L. D. and Lifshitz, E. M.},
  edition = 3,
  publisher = {Butterworth-Heinemann},
  title = {Mechanics, Third Edition: Volume 1 (Course of Theoretical Physics)},
  year = 1976
}

@book{lee2018,
  author    = {Lee, John M.},
  title     = {Introduction to {R}iemannian Manifolds},
  publisher = {Springer},
  series    = {Graduate Texts in Mathematics},
  volume    = {176},
  edition   = {2nd},
  year      = {2018}
}

@inproceedings{lipman2023,
  author    = {Lipman, Yaron and Chen, Ricky T. Q. and Ben-Hamu, Heli and Nickel, Maximilian and Le, Matt},
  title     = {Flow Matching for Generative Modeling},
  booktitle = {International Conference on Learning Representations},
  year      = {2023}
}

@inproceedings{liu2023,
  title={Flow Straight and Fast: Learning to Generate and Transfer Data with Rectified Flow},
  author={Liu, Xingchao and Gong, Chengyue and Liu, Qiang},
  booktitle={International Conference on Learning Representations},
  year={2023},
  url={https://openreview.net/forum?id=XVjTT1nw5z}
}

@inproceedings{lou2020,
  author    = {Lou, Aaron and Lim, Derek and Katsman, Isay and Huang, Leo and Jiang, Qingxuan and
               Lim, Ser-Nam and De Sa, Christopher},
  title     = {Neural Manifold Ordinary Differential Equations},
  booktitle = {Advances in Neural Information Processing Systems 33},
  year      = {2020}
}

@book{mardia2000,
  author    = {Mardia, Kanti V. and Jupp, Peter E.},
  title     = {Directional Statistics},
  publisher = {Wiley},
  series = {Wiley Series in Probability and Statistics},
  year      = {2000}
}

@inproceedings{mathieu2020,
  author    = {Mathieu, Emile and Nickel, Maximilian},
  title     = {{R}iemannian Continuous Normalizing Flows},
  booktitle = {Advances in Neural Information Processing Systems 33},
  pages     = {8205--8216},
  year      = {2020}
}

@inproceedings{nagano2019,
  author    = {Nagano, Yoshihiro and Yamaguchi, Shoichiro and Fujita, Yasuhiro and Koyama, Masanori},
  title     = {A Wrapped Normal Distribution on Hyperbolic Space for Gradient-Based Learning},
  booktitle = {International Conference on Machine Learning},
  pages     = {4693--4702},
  year      = {2019}
}

@inproceedings{nickel2017,
  title={Poincaré Embeddings for Learning Hierarchical Representations},
  author={Nickel, Maximilian and Kiela, Douwe},
  booktitle={Advances in Neural Information Processing Systems (NIPS)},
  pages={6341--6350},
  year={2017}
}

@incollection{neal2011,
  author    = {Neal, Radford M.},
  title     = {{MCMC} Using {H}amiltonian Dynamics},
  booktitle = {Handbook of {M}arkov Chain {M}onte {C}arlo},
  publisher = {Chapman \& Hall/CRC},
  pages     = {113--162},
  year      = {2011}
}

@article{papamakarios2021,
  author  = {Papamakarios, George and Nalisnick, Eric and Rezende, Danilo J. and
             Mohamed, Shakir and Lakshminarayanan, Balaji},
  title   = {Normalizing Flows for Probabilistic Modeling and Inference},
  journal = {Journal of Machine Learning Research},
  volume  = {22},
  number  = {57},
  pages   = {1--64},
  year    = {2021}
}

@inproceedings{perezcruz2008,
  author    = {P\'erez-Cruz, Fernando},
  title     = {Kullback--Leibler divergence estimation of continuous distributions},
  booktitle = {IEEE International Symposium on Information Theory (ISIT)},
  year      = {2008},
  pages     = {1666--1670},
}

@inproceedings{rezende2015,
  author    = {Rezende, Danilo J. and Mohamed, Shakir},
  title     = {Variational Inference with Normalizing Flows},
  booktitle = {International Conference on Machine Learning},
  pages     = {1530--1538},
  year      = {2015}
}

@inproceedings{rezende2020,
  author    = {Rezende, Danilo J. and Papamakarios, George and Racani\`ere, S\'ebastien and
               Albergo, Michael S. and Kanwar, Gurtej and Shanahan, Phiala E. and Cranmer, Kyle},
  title     = {Normalizing Flows on Tori and Spheres},
  booktitle = {International Conference on Machine Learning},
  pages     = {8083--8092},
  year      = {2020}
}

@book{robert2004,
  title={Monte Carlo Statistical Methods},
  author={Robert, Christian P. and Casella, George},
  edition={Second},
  year={2004},
  publisher={Springer},
  address={New York},
  doi={10.1007/978-1-4757-4145-2}
}

@book{rouviere2016,
  author    = {François Rouvière},
  title     = {Initiation à la géométrie de Riemann},
  publisher = {Calvage \& Mounet},
  year      = {2016}
}

@inproceedings{rozen2021,
  author    = {Rozen, Noam and Grover, Aditya and Nickel, Maximilian and Lipman, Yaron},
  title     = {Moser Flow: Divergence-Based Generative Modeling on Manifolds},
  booktitle = {Advances in Neural Information Processing Systems 34},
  year      = {2021}
}

@inproceedings{sohl-dickstein2015,
  title = 	 {Deep Unsupervised Learning using Nonequilibrium Thermodynamics},
  author = 	 {Sohl-Dickstein, Jascha and Weiss, Eric and Maheswaranathan, Niru and Ganguli, Surya},
  booktitle = 	 {Proceedings of the 32nd International Conference on Machine Learning},
  pages = 	 {2256--2265},
  year = 	 {2015},
  editor = 	 {Bach, Francis and Blei, David},
  volume = 	 {37},
  series = 	 {Proceedings of Machine Learning Research},
  address = 	 {Lille, France},
  month = 	 {07--09 Jul},
  publisher =    {PMLR},
  url = 	 {https://proceedings.mlr.press/v37/sohl-dickstein15.html}
}

@incollection{sokal1997,
  title={Monte {Carlo} methods in statistical mechanics: foundations and new algorithms},
  author={Sokal, Alan},
  booktitle={Functional Integration: Basics and Applications},
  series={NATO ASI Series B},
  volume={361},
  pages={131--192},
  publisher={Springer},
  year={1997}
}

@inproceedings{song2021,
  author    = {Song, Yang and Sohl-Dickstein, Jascha and Kingma, Diederik P. and
               Kumar, Abhishek and Ermon, Stefano and Poole, Ben},
  title     = {Score-Based Generative Modeling through Stochastic Differential Equations},
  booktitle = {International Conference on Learning Representations},
  year      = {2021}
}

@inproceedings{souveton2024,
  author    = {Souveton, Vincent and Guillin, Arnaud and Jasche, Jens and Lavaux, Guilhem and Michel, Manon},
  title     = {Fixed-Kinetic Neural {H}amiltonian Flows for Enhanced Interpretability
               and Reduced Complexity},
  booktitle = {International Conference on Artificial Intelligence and Statistics},
  pages     = {3178--3186},
  year      = {2024}
}

@article{tabak2010,
  author  = {Tabak, Esteban G. and Vanden-Eijnden, Eric},
  title   = {Density Estimation by Dual Ascent of the Log-Likelihood},
  journal = {Communications in Mathematical Sciences},
  volume  = {8},
  number  = {1},
  pages   = {217--233},
  year    = {2010}
}

@inproceedings{toth2020,
  author    = {Toth, Peter and Rezende, Danilo J. and Jaegle, Andrew and Racani\`ere, S\'ebastien and
               Botev, Aleksandar and Higgins, Irina},
  title     = {Hamiltonian Generative Networks},
  booktitle = {International Conference on Learning Representations},
  year      = {2020}
}

@book{binney2008,
  author    = {Binney, James and Tremaine, Scott},
  title     = {Galactic Dynamics},
  edition   = {2nd},
  publisher = {Princeton University Press},
  year      = {2008}
}

@article{kirkwood1935,
  author  = {Kirkwood, John G.},
  title   = {Statistical Mechanics of Fluid Mixtures},
  journal = {The Journal of Chemical Physics},
  volume  = {3},
  number  = {5},
  pages   = {300--313},
  year    = {1935}
}

@inproceedings{rasmussen2003,
  title={Gaussian processes to speed up hybrid {M}onte {C}arlo for expensive {B}ayesian integrals},
  author={Rasmussen, Carl Edward},
  booktitle={Bayesian Statistics 7},
  pages={651--659},
  publisher={Oxford University Press},
  year={2003}
}

@article{zhang2017,
  title={Hamiltonian {M}onte {C}arlo acceleration using surrogate functions with random bases},
  author={Zhang, Cheng and Shahbaba, Babak and Zhao, Hongkai},
  journal={Statistics and Computing},
  volume={27},
  number={6},
  pages={1473--1490},
  year={2017}
}

@article{li2019,
  title={Neural network gradient {H}amiltonian {M}onte {C}arlo},
  author={Li, Lingge and Holbrook, Andrew and Shahbaba, Babak and Baldi, Pierre},
  journal={Computational Statistics},
  volume={34},
  number={1},
  pages={281--299},
  year={2019}
}

\appendix

\section{Background on Riemannian geometry}
\label{app:geometry}

We give formal definitions of the tools of Section~\ref{sec:hamiltonian} and refer to \citet{docarmo1992, lee2018, rouviere2016} for details. Table~\ref{tab:primitives} gives these tools in closed form on our three model spaces. $(\M,g)$ is a smooth, complete, connected $d$-dimensional Riemannian manifold. In local coordinates the metric has matrix $G(q)$, a covector $p$ corresponds to the velocity $p^\sharp=G(q)^{-1}p$, and $\normg{p}^2=p^\top G(q)^{-1}p$. The Riemannian volume is $\dvol(q)=\sqrt{\det G(q)}\,\dd q$. For a smooth function $V$, $\dd V$ is its differential, $\gradg V=G^{-1}\dd V$ its gradient, $\Hess V$ its Hessian, and $\Delta_gV=\divg\gradg V$ its Laplace--Beltrami operator. In local coordinates, Hamilton's equations for~\eqref{eq:H} read $\dot q=G(q)^{-1}p$ and $\dot p=-\tfrac12\partial_q\bigl(p^\top G(q)^{-1}p\bigr)-\dd V(q)$, which is equivalent to the covariant Newton equation $\nabla_t\dot q=-\gradg V$. The exponential map, parallel transport and geodesic distance are those of the Levi-Civita connection. On a normal neighborhood, the volume Jacobian of the exponential map is defined by $(\exp_q)^*\dvol=J_q(v)\,\dd v$, where $\dd v$ is the Euclidean volume induced by $g_q$ on $T_q\M$, and the law of $Q=\exp_o(X)$ with $X\sim\mu$ has density $\mu(x)/J_o(x)$ at $\exp_o(x)$ whenever $\exp_o$ is injective on the support of $\mu$. On $\Hyp^2$, the exponential map at any point is a global diffeomorphism. On $\Sph^2$, it has a cut locus at the antipode. The von Mises--Fisher density on $\Sph^2$ is $\mathrm{vMF}(q;\mu,\kappa)=c(\kappa)e^{\kappa\mu^\top q}$ with $c(\kappa)=\kappa/(4\pi\sinh\kappa)$ \citep{mardia2000}.

\begin{table}[h]
\centering
\caption{Geometric primitives on the three model spaces. For $\Hyp^2$, we use the hyperboloid model in $\R^3$ with the Lorentzian product $\langle x,y\rangle_{\mathcal L}=-x_0y_0+x_1y_1+x_2y_2$. Here $r=\lVert v\rVert$ and $\hat v=v/r$.}
\label{tab:primitives}
\small
\begin{tabular}{l ccc}
\toprule
 & $\R^2$ & $\Hyp^2$ & $\Sph^2$ \\
\midrule
Curvature & $0$ & $-1$ & $+1$ \\
Exponential map $\exp_q(v)$ & $q+v$ & $\cosh(r)q+\sinh(r)\hat v$ & $\cos(r)q+\sin(r)\hat v$ \\
Geodesic distance & $\lVert q-q'\rVert$ & $\operatorname{arccosh}(-\langle q,q'\rangle_{\mathcal L})$ & $\arccos(q^\top q')$ \\
Volume Jacobian $J_q(v)$ & $1$ & $\sinh r/r$ & $\sin r/r$ \\
Position base $\rho_0$ & Gaussian & wrapped Gaussian & uniform \\
\bottomrule
\end{tabular}
\end{table}

\section{The Jeans equations}
\label{app:moments}

Let $f_t$ be the phase-space density of the gas, $v=p^\sharp$, and define $\rho_t=\int f_t\,\dd\nu_q$, $\rho_tu_t=\int v\,f_t\,\dd\nu_q$ and $S_t=\int v\otimes v\,f_t\,\dd\nu_q$. Since the flow preserves $\lambda$, $\tfrac{\dd}{\dd t}\int Af_t\,\dd\lambda=\int\tfrac{\dd A}{\dd t}f_t\,\dd\lambda$ for every observable $A$, where $\tfrac{\dd A}{\dd t}$ is the derivative along trajectories. For $A=\varphi(q)$, with $\varphi$ a smooth compactly supported function, $\tfrac{\dd A}{\dd t}=\langle\gradg\varphi,v\rangle_g$, and integrating over each fibre with $\dd\lambda=\dvol\,\dd\nu_q$, then integrating by parts on $\M$, gives the continuity equation $\partial_t\rho_t+\divg(\rho_tu_t)=0$. For $A=\langle X(q),v\rangle_g$, with $X$ a smooth compactly supported vector field, the covariant Newton equation and $\nabla g=0$ give $\tfrac{\dd A}{\dd t}=\langle\nabla_vX,v\rangle_g-\langle X,\gradg V\rangle_g$, and the same steps yield the Jeans momentum equation of stellar dynamics \citep{binney2008}, $\partial_t(\rho_tu_t)+\divg S_t=-\rho_t\gradg V$. These are the two equations~\eqref{eq:jeans}, which are exact.

When the velocities are isotropic around their mean with variance $\tau$, that is, $\int(v-u_t)\otimes(v-u_t)\,f_t\,\dd\nu_q=\tau\rho_tg^{-1}$, the momentum flux is $S_t=\rho_t\,u_t\otimes u_t+\tau\rho_tg^{-1}$, and $\divg(\tau\rho_tg^{-1})=\tau\gradg\rho_t$ by metric compatibility. Combined with the continuity equation, the Jeans equation then becomes the Euler equation of an isothermal gas, $\rho_t(\partial_tu_t+\nabla_{u_t}u_t)=-\tau\gradg\rho_t-\rho_t\gradg V$. A gas at rest, $u_t=0$, is stationary if and only if $\tau\gradg\rho=-\rho\gradg V$, that is, $\gradg(\log\rho+V/\tau)=0$, or $\rho\propto e^{-V/\tau}$ on the connected manifold $\M$. For the initial lift $f_0=\rho_0M_\tau$, odd moments vanish and $\int v\otimes v\,M_\tau\,\dd\nu_q=\tau g^{-1}$, so $u_0=0$ and $S_0=\tau\rho_0g^{-1}$. At $t=0$, $\partial_t(\rho_tu_t)=\rho_0\partial_tu_t$ because $u_0=0$, and dividing by $\rho_0>0$ gives $\partial_tu_t|_0=-\gradg V-\tau\gradg\log\rho_0$, which is~\eqref{eq:acceleration} since $-\gradg V=\tau\gradg\log\Pv$.

\section{Proof of Proposition~\ref{prop:decomposition}}
\label{app:decomposition}

\begin{lemma}[Time reversibility]\label{lem:reversal}
Let $\mathcal R(q,p)=(q,-p)$. Then $\Flow_{-t}=\mathcal R\circ\Flow_t\circ\mathcal R$ for all $t$.
\end{lemma}

\begin{proof}
Since $H$ is even in $p$, if $(q(t),p(t))$ solves Hamilton's equations, so does $(q(-t),-p(-t))$, and uniqueness of solutions gives the claim.
\end{proof}

Write $\pr$ for the projection on positions and $Q_{-T}=\pr\Flow_{-T}(q,p)$. The flow preserves $\lambda$, so the phase-space density at time $T$ is $f_T=f_0\circ\Flow_{-T}$, that is, $f_T(q,p)=\rho_0(Q_{-T})\,M_\tau(\Flow_{-T}(q,p))$. The momentum density depends on the kinetic energy only, $M_\tau=(2\pi\tau)^{-d/2}e^{-K/\tau}$ with $K=\tfrac12\normg{p}^2$, and energy conservation gives $K(\Flow_{-T}(q,p))=K(q,p)+V(q)-V(Q_{-T})$. Hence
\begin{equation*}
  f_T(q,p)=M_\tau(q,p)\,e^{-V(q)/\tau}\,\rho_0(Q_{-T})\,e^{V(Q_{-T})/\tau}
  =M_\tau(q,p)\,\Pv(q)\,w(Q_{-T}),
\end{equation*}
since the normalizing constants cancel. Integrating over the momenta, distributed as $\Ng(0,\tau I)$, gives $\rho_T(q)=\Pv(q)\,\E_P[w(\pr\Flow_{-T}(q,P))]$. By Lemma~\ref{lem:reversal}, $\pr\Flow_{-T}(q,P)=\pr\Flow_T(q,-P)$, which has the same law as $\pr\Flow_T(q,P)$ by symmetry of the Gaussian. This proves~\eqref{eq:decomposition}. Taking logarithms gives $V=-\tau\log\rho_T+\tau\log R_T+\mathrm{cst}$, and subtracting $V_\star=-\tau\log\rho_1$ yields~\eqref{eq:gap}.

\section{Proof of Theorem~\ref{thm:thermalization}}
\label{app:certificate}

\begin{proof}[Proof of Theorem~\ref{thm:thermalization}]
The phase-space Boltzmann distribution $\mu_V=\Pv M_\tau$ has density proportional to $e^{-H/\tau}$ with respect to $\lambda$. Since the flow preserves $\lambda$ and $H$, $\mu_V$ is invariant, and projecting on positions gives $\Pv\Kk=\Pv$. Independence of $Q_0$ and $Q_T$ means that $\Kk(q,\cdot)=\rho_T$ for $\rho_0$-almost every $q$, hence for $\Pv$-almost every $q$ since $\rho_0>0$. Therefore $\Pv=\Pv\Kk=\int\Pv(\dd q)\,\Kk(q,\cdot)=\rho_T$. If moreover $\rho_T=\rho_1$, then $e^{-V/\tau}\propto\rho_1$, that is, $V=V_\star+\mathrm{cst}$.
\end{proof}

\section{The RCNF baseline}
\label{app:rcnf}

The baseline is a Riemannian continuous normalizing flow \citep{mathieu2020}, re-implemented with the same geometric pipeline as RNHF, so that the two models differ only in the flow family and the training objective. Particles evolve on $\M$ according to $\dot q_t=f_\theta(t,q_t)\in T_{q_t}\M$ for $t\in[0,1]$, where $f_\theta$ is a neural network whose output is projected onto the tangent space at the current point. The equation is integrated with fixed-size geodesic steps $q_{t+h}=\exp_{q_t}(h\tilde f)$, where $\tilde f$ is the geodesic midpoint estimate of the vector field, intermediate evaluations being parallel-transported back to the current tangent space. The midpoint rule matches the second-order accuracy of the geodesic integrator, but the scheme has no phase space and no symplectic structure. The original implementation of \citet{mathieu2020} uses an adaptive projective Runge--Kutta solver in ambient coordinates, whereas our fixed-step, re-anchored variant is equivalent in the continuous-time limit and makes the discrete geometry identical to that of RNHF. The log-density follows the instantaneous change-of-variables formula $\tfrac{\dd}{\dd t}\log p_t(q_t)=-\divg f_\theta(t,q_t)$. In normal coordinates centered at the current point, the metric correction to the divergence vanishes, so the divergence reduces to the trace of $\partial f/\partial x$ in an orthonormal frame, which we compute exactly by differentiating $f_\theta$ along the two frame directions. In higher dimension it can be estimated with Hutchinson's estimator \citep{grathwohl2019}. The RCNF velocity field is capped on $\R^2$ and $\Hyp^2$ to stabilize the fixed-step likelihood integration, while no cap is used for RNHF. The base is the same as for RNHF, and the model is trained by maximum likelihood, at the cost of one divergence per integration step and, during training, of its gradient, which involves mixed second derivatives of $f_\theta$.

\section{Details of the experiments on the three model spaces}
\label{app:setting}

RCNF is trained by maximum likelihood and evaluates a Riemannian divergence at every integration step. RNHF avoids this divergence but optimizes the variational bound~\eqref{eq:elbo}. Both models have approximately 34k parameters, but RNHF learns a potential, a momentum encoder, and the global momentum scale, whereas RCNF learns only a velocity field. Both models are trained for 120 epochs and optimized with the Adam algorithm with learning rate $10^{-3}$, minibatches of 256, and a fixed training set of 6,000 samples. RNHF uses $L=30$ numerical steps and RCNF integrates over the same interval $[0,1]$ with the same number of second-order geodesic midpoint steps. We report six sample-based criteria: Kullback--Leibler (KL) divergence, ambient mean discrepancy, ambient covariance discrepancy, energy distance based on geodesic distances, and sample precision and recall. Precise definitions are given in Appendix~\ref{app:metrics}. 

On $\mathbb R^2$, the base is the centered Gaussian $\mathcal N(0,1.4^2I_2)$ and the target is a bimodal Gaussian mixture whose modes are symmetric about the origin with unequal weights,
\[
  \rho_{\mathrm{data}}(x)
  =\tfrac35\mathcal N(x;m_1,\sigma_{\mathrm{data}}^2I_2)
   +\tfrac25\mathcal N(x;m_2,\sigma_{\mathrm{data}}^2I_2),
  \qquad m_{1,2}=(\mp2,0),\quad \sigma_{\mathrm{data}}=0.5 .
\]
On $\mathbb H^2$, the base is a wrapped Gaussian at the origin of scale $0.8$ and the target is a mixture of two wrapped Gaussians whose modes are symmetric about the origin with unequal weights. In the tangent space at the origin, with $a_{1,2}=(\mp0.75,0)$ and $\sigma_{\mathrm{data}}=0.2$,
\[
  X\sim\tfrac35\mathcal N(a_1,\sigma_{\mathrm{data}}^2I_2)
      +\tfrac25\mathcal N(a_2,\sigma_{\mathrm{data}}^2I_2),
  \qquad Q=\exp_o(X).
\]
On $\mathbb S^2$, the base is uniform and the target is a two-component von Mises--Fisher mixture with unequal weights. Calling $\mu_1$ the north pole and $\mu_2$ the location at a polar angle of $110^\circ$ in the $xz$-plane, 
\[
  \rho_{\mathrm{data}}(q)
  =\tfrac35\mathrm{vMF}(q;\mu_1,\kappa)
   +\tfrac25\mathrm{vMF}(q;\mu_2,\kappa),
  \qquad \kappa=20.
\]

\section{Evaluation metrics}
\label{app:metrics}

Each sample-based metric is reported in the chart-free space. Distance-based metrics use the exact geodesic distance $d_g$: the arc length $\arccos \langle q, q' \rangle$ on $\Sph^d$, the Lorentzian formula $\operatorname{arccosh}(-\langle q, q' \rangle_{\mathrm{L}})$ on $\Hyp^d$, and the Euclidean distance on $\R^d$. Moment-based metrics use the ambient embedding ($\R^3$ for $\Sph^2$, the hyperboloid in $\R^3$ for $\Hyp^2$). Below, $X = \{x_i\}_{i=1}^{n}$ are training samples and $Y = \{y_j\}_{j=1}^{m}$ generated samples, and $\|\cdot\|$ denotes the relevant distance. Lower is better for all metrics except precision and recall.

\paragraph{Kullback--Leibler divergence.}
We estimate the KL divergence $\mathrm{KL}(\rho_1 \,\|\, p_\theta)$ between the target and each model from samples alone, using the $k$-nearest-neighbor estimator of \citet{perezcruz2008}. Given $n$ target samples $\{x_i\}$ and $m$ model samples $\{y_j\}$ on a $d$-dimensional manifold, the estimator reads
\[
  \widehat{\mathrm{KL}}(\rho_1 \| p_\theta)
  = \frac{d}{n}\sum_{i=1}^{n} \log \frac{\nu_k(x_i)}{\rho_k(x_i)}
  + \log \frac{m}{n-1},
\]
where $\rho_k(x_i)$ is the geodesic distance from $x_i$ to its $k$-th nearest neighbour among the other target samples, and $\nu_k(x_i)$ the geodesic distance to its $k$-th nearest neighbour among the model samples. It requires no closed-form model density and therefore compares the two models symmetrically, using $4{,}000$ samples from each. We use $k=1$ and average over three independent draws.

\paragraph{Mean discrepancy.} With $\bar x = \frac{1}{n} \sum_{i=1}^{n} x_i$ and $\bar y = \frac{1}{m} \sum_{j=1}^{m} y_j$,
\[
  \mathrm{mean}_{\ell_2}(X, Y) \;=\; \|\bar x - \bar y\|_2 .
\]

\paragraph{Covariance discrepancy.} With $\widehat\Sigma_X = \frac{1}{n-1} \sum_{i=1}^{n} (x_i - \bar x)(x_i - \bar x)^\top$ and similarly $\widehat\Sigma_Y$,
\[
  \mathrm{cov}_{\mathrm{Fro}}(X, Y) \;=\; \big\| \widehat\Sigma_X - \widehat\Sigma_Y \big\|_F .
\]

\paragraph{Energy distance.} The empirical energy distance,
\[
  \widehat{\mathcal{E}}(X, Y)
  \;=\;
  \frac{2}{nm} \sum_{i=1}^{n} \sum_{j=1}^{m} \|x_i - y_j\|
  \;-\;
  \frac{1}{n^2} \sum_{i=1}^{n} \sum_{i'=1}^{n} \|x_i - x_{i'}\|
  \;-\;
  \frac{1}{m^2} \sum_{j=1}^{m} \sum_{j'=1}^{m} \|y_j - y_{j'}\|,
\]
compares cross-distances between the two samples with the internal distances within each. 

\paragraph{Precision and recall at the training radius.} A reference radius is computed from the training data only: with $r^{(k)}_i$ the distance from $x_i$ to its $k$-th nearest neighbor in $X \setminus \{x_i\}$ (we use $k = 5$), $r \;=\; \operatorname{median}_{1 \le i \le n}\; r^{(k)}_i$, and
\begin{align*}
  \mathrm{Precision}(X,Y)
  &=\frac{1}{m}\sum_{j=1}^{m}\mathbf{1}\left[\min_{1\le i\le n}\|y_j-x_i\|\le r\right],\\
  \mathrm{Recall}(X,Y)
  &=\frac{1}{n}\sum_{i=1}^{n}\mathbf{1}\left[\min_{1\le j\le m}\|x_i-y_j\|\le r\right].
\end{align*}
High precision means that generated samples fall in regions supported by the training data. High recall means that they cover most of the training support. The same radius $r$ is used for both models.

\section{Details of the IceCube experiment}
\label{app:icecube}
 
\paragraph{Data and injection.} This work uses the public IceCube point-source data release for the years 2008--2018 \citep{icecube2021data}, described in \citet{icecube2021} and analyzed in \citet{icecube2020prl}. We use the event files of the ten seasons of the public sample (IC40 to IC86-VII), that is, $1{,}134{,}450$ events, and from each event only its direction and its angular uncertainty. The right ascensions are scrambled uniformly once, with a fixed seed, and the same scrambled data set is used by all runs. The $n_s=100$ signal events are drawn from $\mathrm{vMF}(n_{\mathrm{src}},\sigma_i^{-2})$ around NGC~1068, with $\sigma_i$ drawn from the events whose $\sin\delta$ differs from that of the source by less than $0.05$. The background density is $B(x)=h(\sin\delta)/(2\pi)$, where $h$ is a histogram of $\sin\delta$ over all events with $500$ bins, smoothed by a Gaussian kernel of one bin. 
 
\paragraph{Maximum likelihood estimation.} We search the maximum of the likelihood for declinations below $85^\circ$ in absolute value, on a grid of $20{,}000$ points followed by a gradient descent from the ten best points.
 
\paragraph{Samplers.} Exact Hamiltonian Monte Carlo uses $L=20$ steps, a burn-in of $500$ transitions during which the duration~\eqn{eq:ell} is re-estimated for $E$ every $100$ transitions, and then $10{,}000$ states, of which the first $1{,}000$ train RNHF. The latter uses the hyperparameters of Section~\ref{app:setting} and is trained for $200$ epochs with Adam, learning rate $10^{-3}$ and minibatches of $256$. Its networks receive the tangent coordinates $\log_c(q)$ at the mean direction $c$ of the training states, divided by their root-mean-square angular spread $a$. Its base is a wrapped Gaussian at $c$ with standard deviation $2a/\sqrt2$ per coordinate, and $\sigma_p$ is initialized at the clock of the quarter turn, $(\pi/2)\,a/\sqrt2$. The exponential map and the parallel transport are those of $\Sph^2$. The duration of the surrogate sampler is computed with~\eqn{eq:ell} on the $1{,}000$ training states. The long reference merges two independent chains of exact Hamiltonian Monte Carlo of $10{,}000$ states each.
 
\paragraph{Measurements.} Standard errors are computed by batch means over $20$ batches, and integrated autocorrelation times of $E$ with the automatic window of \citet{sokal1997}, with $c=5$. One evaluation of $E$ or of its gradient takes $17.9$~ms on average and one gradient of the network $0.63$~ms. The break-even point~\eqn{eq:breakeven} can be written $N^\star=\frac{T_0}{\tau_H(L+1)\,t_E-\tau_S\,(t_E+L\,t_{\mathrm{net}})}$, where $t_E$ and $t_{\mathrm{net}}$ are the costs of one evaluation of $E$ and of one gradient of the network, and $\tau_H$ and $\tau_S$ the autocorrelation times of the two chains. Table~\ref{tab:icecube} gives the results of each seed.
 
\begin{table}[h]
\centering
\caption{Results of the IceCube experiment for each seed, over the $9{,}000$ samples of the second phase (HMC: exact Hamiltonian Monte Carlo; S: surrogate sampler). Angles are in degrees, costs per effective sample in seconds, and $T_0$ is the cost of training RNHF and computing $\ell$. The long reference gives an angle to the source of $0.0709\pm0.0005$ and a $90\%$ radius of $0.1453\pm0.0007$. The slope and correlation are those of $-\log U_\theta$ against $E$ over the $1{,}000$ training states.}
\label{tab:icecube}
\small
\begin{tabular}{lccc}
\toprule
 & seed 0 & seed 1 & seed 2 \\
\midrule
angle to source, HMC / S & $0.0705$ / $0.0714$ & $0.0708$ / $0.0696$ & $0.0714$ / $0.0717$ \\
$90\%$ radius, HMC / S & $0.1452$ / $0.1443$ & $0.1427$ / $0.1440$ & $0.1451$ / $0.1453$ \\
acceptance, HMC / S & $0.999$ / $0.917$ & $0.999$ / $0.909$ & $0.999$ / $0.945$ \\
$\tau$, HMC / S & $1.07$ / $7.32$ & $0.96$ / $2.40$ & $1.09$ / $6.27$ \\
cost per effective sample, HMC / S & $0.395$ / $0.221$ & $0.370$ / $0.073$ & $0.412$ / $0.196$ \\
$T_0$ & $308.7$ & $307.8$ & $308.3$ \\
break-even $N^\star$ & $1{,}773$ & $1{,}036$ & $1{,}426$ \\
slope / correlation & $1.11$ / $0.990$ & $0.96$ / $0.972$ & $1.08$ / $0.995$ \\
\bottomrule
\end{tabular}
\end{table}

\end{document}